\documentclass[sigconf]{acmart}

\usepackage{amssymb}
\usepackage{amsmath}
\usepackage{algorithmic}
\usepackage{algorithm}
\usepackage{bm}
\usepackage{multirow}
\usepackage{textcomp}
\usepackage{booktabs}
\newcommand{\x}{\mathbf{x}}

\usepackage{url}
\usepackage{breakurl}

\usepackage{color}
\usepackage{flushend}
\usepackage{mathtools}
\usepackage{bbm}
\allowdisplaybreaks

\usepackage{graphicx}

\newcommand{\tabincell}[2]{\begin{tabular}{@{}#1@{}}#2\end{tabular}}

\AtBeginDocument{%
  \providecommand\BibTeX{{%
    \normalfont B\kern-0.5em{\scshape i\kern-0.25em b}\kern-0.8em\TeX}}}

\setcopyright{iw3c2w3}
\copyrightyear{2021}
\acmYear{2021}
\acmConference[WWW '21]{Proceedings of the Web Conference 2021}{April 19--23, 2021}{Ljubljana, Slovenia}
\acmBooktitle{Proceedings of the Web Conference 2021 (WWW '21), April 19--23, 2021, Ljubljana, Slovenia}
\acmPrice{}
\acmDOI{10.1145/3442381.3449884}
\acmISBN{978-1-4503-8312-7/21/04}

\begin{document}

\title{Hashing-Accelerated Graph Neural Networks for Link Prediction}


\author{Wei Wu}
\affiliation{%
  \institution{L3S Research Center, Leibniz University Hannover}
  \city{Hannover}
  \country{Germany}}
\email{william.third.wu@gmail.com}

\author{Bin Li}
\authornote{Corresponding author.}
\affiliation{%
  \institution{School of Computer Science, Fudan University}
  \city{Shanghai}
  \country{China}
}
\email{libin@fudan.edu.cn}

\author{Chuan Luo}
\affiliation{%
 \institution{Microsoft Research}
 \city{Beijing}
 \country{China}}
\email{chuan.luo@microsoft.com}

\author{Wolfgang Nejdl}
\affiliation{%
  \institution{L3S Research Center, Leibniz University Hannover}
  \city{Hannover}
  \country{Germany}}
\email{nejdl@l3s.de}

\renewcommand{\shortauthors}{Wei Wu, Bin Li, Chuan Luo and Wolfgang Nejdl}

\begin{abstract}
Networks are ubiquitous in the real world. Link prediction, as one of the key problems for network-structured data, aims to predict whether there exists a link between two nodes. The traditional approaches are based on the explicit similarity computation between the compact node representation by embedding each node into a low-dimensional space. In order to efficiently handle the intensive similarity computation in link prediction, the hashing technique has been successfully used to produce the node representation in the Hamming space. However, the hashing-based link prediction algorithms face accuracy loss from the randomized hashing techniques or inefficiency from the learning to hash techniques in the embedding process. Currently, the Graph Neural Network (GNN) framework has been widely applied to the graph-related tasks in an end-to-end manner, but it commonly requires substantial computational resources and memory costs due to massive parameter learning, which makes the GNN-based algorithms impractical without the help of a powerful workhorse. In this paper, we propose a simple and effective model called \#GNN, which balances the trade-off between accuracy and efficiency. \#GNN is able to efficiently acquire node representation in the Hamming space for link prediction by exploiting the randomized hashing technique to implement message passing and capture high-order proximity in the GNN framework. Furthermore, we characterize the discriminative power of \#GNN in probability. The extensive experimental results demonstrate that the proposed \#GNN algorithm achieves accuracy comparable to the learning-based algorithms and outperforms the randomized algorithm, while running significantly faster than the learning-based algorithms. Also, the proposed algorithm shows excellent scalability on a large-scale network with the limited resources.
\end{abstract}




\maketitle

\section{Introduction}

Networks are ubiquitous in the real world, for example, the social network, the co-authorship network and the World Wide Web, etc., which has also fostered the network mining research. A significant research topic, \emph{link prediction}, is to predict whether there exists a link between two nodes. Furthermore, it underpins many high-level applications, e.g., friend recommendation in social networks \cite{adamic2003friends}, product recommendation in e-commerce \cite{koren2009matrix}, finding interactions between proteins \cite{nickel2015review}, metabolic network reconstruction \cite{oyetunde2017boostgapfill}, knowledge graph completion \cite{nickel2015review}, etc. In real-world scenarios, one network contains not only complex structure information between nodes but also rich content information carried by the nodes, both of which facilitate us to understand the networks better. For example, a protein-protein interaction network is composed of proteins as nodes where physico-chemical properties of the proteins are attributes and interaction between proteins as edges; in a follower-followee social network, nodes denoting users contain many profiles. Although many algorithms can be used for link prediction, e.g., DeepWalk \cite{perozzi2014deepwalk}, node2vec \cite{grover2016node2vec}, LINE \cite{tang2015line}, NetMF \cite{qiu2018network}, INH-MF \cite{lian2018high}, NodeSketch \cite{yang2019nodesketch} and WLNM \cite{zhang2017weisfeiler}, etc., the above methods capture only structure information of the network, which means that important attribute information is simply ignored. Therefore, it is of significance to develop link prediction algorithms which can simultaneously preserve attribute and structure information of the network.

Thus far much effort has been devoted to the link prediction algorithms on attributed networks. The traditional approaches are based on the explicit similarity computation between the simplified node representation by embedding each node into a low-dimensional space. For example, the algorithms based on cosine similarity such as TADW \cite{yang2015network}, HSCA \cite{zhang2016homophily} and CANE \cite{tu2017cane} preserve both attribute and structure information in the Euclidean space. However, it is insufficiently efficient to calculate the cosine similarity for link prediction which heavily involves the similarity computation \cite{lian2018high,yang2019nodesketch}. In order to address the issue, the hashing techniques including the learning to hash methods such as BANE \cite{yang2018binarized} and LQANR \cite{yang2019low} and the randomized hashing ones such as NetHash \cite{wu2018efficient} have been applied to embed the nodes into the Hamming space. Consequently, the rapid Hamming similarity computation remarkably speedups link prediction \cite{lian2018high,yang2019nodesketch}. Despite that, in the embedding process, the former is usually time-consuming due to hash code learning, while the latter pursues high efficiency at the sacrifice of precision. 

Currently, the popular solution to the graph-related tasks is the Graph Neural Network (GNN) framework, which effectively learns the hidden patterns and acquires node embedding with high representational power by propagating and aggregating information in the Message Passing scheme \cite{duvenaud2015convolutional,li2016gated,kearnes2016molecular,schutt2017quantum,gilmer2017neural,lei2017deriving,zhou2019scenegraphnet,zhang2020dynamic}. In particular, Weisfeiler-Lehman Kernel Neural Network (WLKNN) \cite{lei2017deriving} can iteratively capture high-order proximity for graph classification by simulating the process of node relabeling in the Weisfeiler-Lehman (WL) graph kernel \cite{shervashidze2011weisfeiler}. Furthermore, many GNN-based algorithms are proposed to conduct the link prediction task in an end-to-end manner, e.g., SEAL \cite{zhang2018link}, GraphSAGE \cite{hamilton2017inductive} and P-GNN \cite{you2019position}. However, the methods require expensive computational resources and memory costs due to massive parameter learning, which in turn hinders the application of these algorithms without the help of a powerful workhorse.

In this paper, we aim to keep a balance between accuracy and efficiency by leveraging the advantages of the hashing techniques and the GNN framework. We observe an interesting connection between WLKNN and MinHash, which is a well-known randomized hashing scheme in the bag-of-words model \cite{broder1998min,wu2018review}, and then are inspired to exploit MinHash to facilitate rapid node representation in WLKNN. Instead of learning weight matrices in WLKNN, we directly employ the random permutation matrices, which are mathematically equivalent with the random permutation operation in MinHash. Consequently, in the WLKNN scheme, the MinHash algorithm iteratively sketches each node and its neighboring nodes and generate node representation, during which the high-order messages are passed along the edges and aggregated within the node. We name the resulting algorithm on the attributed network \#GNN, which efficiently preserves the high-order attribute and structure information for each node in the GNN framework. Considering the theory that a maximally powerful GNN would never map two different substructures to the same representation \cite{xu2018how,morris2019weisfeiler}, we generalize the discriminative power of GNN in probability and further derive the representational power of \#GNN --- it could map two substructures in the graph into the same representation with the probability of their similarity. From the global perspective, after all nodes update the representation, \#GNN starts the next iteration and more importantly, just depends on the present node representation in order to generate the latest representation. Naturally, one produces a Markov chain which is composed of a sequence of attributed networks in the above iteration process. Based on this, \#GNN is scalable w.r.t. the orders of the neighboring nodes, i.e., the number of iterations.

We provide theoretical analysis of the expressive power and conduct extensive empirical evaluation of the proposed \#GNN algorithm and a collection of the state-of-the-art methods, including the hashing-based algorithms and the GNN-based ones, on a number of real-world network datasets. Additionally, we test the scalability of the hashing-based algorithms on a large-scale network with millions of nodes and hundreds of thousands of attributes, and conduct parameter sensitivity analysis of \#GNN. In summary, our contributions are three-fold:
\begin{enumerate}
\item To the best of our knowledge, this work is the first endeavor to introduce the randomized hashing technique into the GNN framework for significant efficiency improvement.
\item We present a scalable link prediction algorithm on the attributed network called \#GNN, which efficiently exploits the randomized hashing technique to capture high-order proximity and to acquire node representation in the GNN framework. Also, we characterize the representational power of \#GNN in probability. 
\item The experimental results show that the proposed \#GNN algorithm achieves accuracy comparable to the learning-based algorithms with significantly reduced runtime (by $2\sim4$ orders of magnitude faster than the GNN-based algorithms). In addition, \#GNN demonstrates excellent scalability on a large-scale network. 
\end{enumerate}

The rest of the paper is organized as follows: Section \ref{sec:work} reviews the related work of graph hashing, graph neural networks and link prediction. Section \ref{sec:pre} introduces the necessary preliminary knowledge. We describe the proposed \#GNN algorithm in Section \ref{sec:algo}. The experimental results are presented in Section \ref{sec:exp}. Finally, we conclude in Section \ref{sec:con}.

\section{Related Work}
\label{sec:work}


\subsection{Graph Hashing}

Hashing techniques have been extensively used to efficiently approximate the high-dimensional data similarity by mapping the similar data instances to the same data points represented as the vectors. The existing hashing techniques consists of two branches: learning to hash and randomized hashing. The former learns the data-specific hash functions to fit the data distribution in the feature space, e.g., Spectral Hashing \cite{weiss2009spectral}, Semantic Hashing \cite{salakhutdinov2009semantic}. 
By contrast, the latter represents the complex data as the compact hash codes by exploiting a family of randomized hash functions, e.g., MinHash \cite{broder1998min}, Weighted MinHash \cite{wu2016canonical,wu2017consistent,wu2018improved}, SimHash \cite{charikar2002similarity,manku2007detecting}, Feature Hashing \cite{weinberger2009feature}.

Hashing techniques have been applied to preserve graph structure information. Discrete Graph Hashing learns high-quality binary hash codes in a discrete optimization framework where the symmetric discrete constraint is designed to preserve the neighborhood structure \cite{liu2014discrete}. By contrast, Asymmetric Discrete Graph Hashing improves performance and lowers the training cost by employing the semantic information and asymmetric discrete constraint \cite{shi2017asymmetric}. In order to further improve efficiency in large-scale problems, Scalable Graph Hashing avoids explicit similarity computation \cite{jiang2015scalable}. In terms of randomized hashing techniques, a 2-dimensional hashing scheme summarizes the edge set and the frequent patterns of co-occurrence edges in the graph streams \cite{aggarwal2011classification}. In \cite{li2012nest, wu2017k}, the randomized hashing algorithms aim to approximately count the tree substructures and convert the graph into a feature vector for graph classification. 

\subsection{Graph Neural Networks}

In recent years, researchers have focused on the Graph Neural Network (GNN) framework, which conducts the graph-related tasks in an end-to-end manner, because it shows the powerful ability to learn hidden patterns in the graph data. The GNN-based algorithms broadly compute node representation in the Message Passing scheme, where information can be propagated from one node to another along edges directly and then aggregated.

Message Passing Neural Network (MPNN) \cite{gilmer2017neural}, which generalizes the Message Passing process in the GNN framework, propagates information further by running multiple message passing iterations for graph classification. As another instance of the Message Passing scheme, Weisfeiler-Lehman Kernel Neural Network (WLKNN) \cite{lei2017deriving} integrates the Weisfeiler-Lehman (WL) graph kernel \cite{shervashidze2011weisfeiler} for propagation and aggregation. However, Xu et.al. point out that the methods are incapable of distinguishing different graph structures to some extent because the different substructures are possibly mapped to the same embedding, and the WL isomorphism test \cite{weisfeiler1968reduction} is the maximal  representational capacity of the GNN models \cite{xu2018how}. Furthermore, Graph Isomorphism Network (GIN) is proposed, which discriminates different graph structures by mapping them to different representations in the embedding space \cite{xu2018how}. In addition to graph classification, Tan et.al. combine the learning to hash technique to the GNN framework for information retrieval \cite{tan2020learning}.

\subsection{Link Prediction}
Link prediction has been popular for the past decades in network analysis \cite{liben2007link}. The traditional link prediction approaches explicitly compute similarity between the nodes of the network, which assumes that two nodes of the network will interact if they are similar in terms of a certain measure, by embedding each node into a low-dimensional space, 
e.g., DeepWalk \cite{mikolov2013distributed}, node2vec \cite{grover2016node2vec}, LINE \cite{tang2015line}, NetMF \cite{qiu2018network}, TADW \cite{yang2015network} and HSCA \cite{zhang2016homophily}. The above algorithms based on cosine similarity are not efficient enough for link prediction which relies on intensive similarity computation \cite{lian2018high,yang2019nodesketch}. To this end, the hashing techniques have been applied because the Hamming similarity computation is superior to the cosine similarity computation in terms of efficiency. INH-MF \cite{lian2018high} and NodeSketch \cite{yang2019nodesketch} preserve only structure information by employing the learning to hash technique and the randomized hashing technique, respectively. Furthermore, BANE \cite{yang2018binarized} and LQANR \cite{yang2019low} learn hash codes to capture attribute and structure information simultaneously, but they are usually expensive in either time or space because they contain massive matrix factorization operations in the learning process. By contrast, NetHash \cite{wu2018efficient} employs the randomized hashing technique to independently sketch the trees rooted at each node and to efficiently preserve attribute and structure information, but it sacrifices precision. On the other hand, it practically explores just limited-order proximity (1st$\sim$3rd order mentioned in \cite{wu2018efficient}), especially in the network with high degrees because each rooted tree must be sketched completely and independently, which makes each node not shared and the time complexity exponential w.r.t. the depth of the rooted trees.

Additionally, the GNN framework has been broadly used for link prediction in an end-to-end manner. WLNM \cite{zhang2017weisfeiler} utilizes only the structure information via subgraph extraction. The process 
requires predefining the size of the subgraphs. Consequently, the truncation operation might give ries to information loss. GraphSAGE \cite{hamilton2017inductive} leverages attribute information and various pooling techniques, and it cannot achieve the maximal representational power because it could embed nodes at symmetric positions into the same embedding. SEAL \cite{zhang2018link} handles with attributes and variable-sized $T$-hop subgraphs together to improve performance. However, the $T$-hop subgraph extraction substantially increases computational cost and memory consumption with $T$ increasing. 
P-GNN \cite{you2019position} further improves the representational power by introducing node position information. Compared with the similarity-based approaches, the GNN-based algorithms are very inefficient in terms of time and space due to massive parameter computation. 

\section{Preliminaries}
\label{sec:pre}

In this section, we introduce some background knowledge in link prediction and methodology, which is helpful to understand the proposed \#GNN method.

\subsection{Link Prediction on an Attributed Network}
Given an attributed network $G=(V,E,f)$, where $V$ denotes the node set of the network, $E$ denotes the undirected edge set of the network, and $f: V \mapsto A$ is a function that assigns the nodes with attributes from an attribute set $A$. A node $v\in V$ with attributes $f(v)$ is represented as a $0/1$ feature vector $\mathbf{v}_v\in \{0, 1\}^{|A|}$, whose dimensions are attributes in $A$. The link prediction task is to predict whether there exists a link between two nodes in the attributed network. In the traditional similarity-based algorithms, link prediction is based on the simplified node representation by embedding each node $v$ as a low-dimensional vector $\x_v \in \mathbb{R}^K$, where $K\ll |A|$, while preserving as much attribute and structure information as possible.

\subsection{Weisfeiler-Lehman Kernel Neural Network}

\begin{figure}[t]
\setlength{\abovecaptionskip}{5pt}%
\setlength{\belowcaptionskip}{0pt}%
\centering
\includegraphics[width=\linewidth]{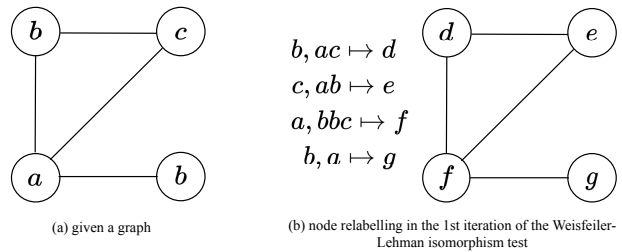}
  \caption{Node relabelling in the Weisfeiler-Lehman isomorphism test. Subplot (a) gives a 4-node graph where each node is assigned with a label, and Subplot (b) shows node relabelling in the 1st iteration.}
\label{fig:wl}
\end{figure}

Weisfeiler-Lehman Kernel Neural Network (WLKNN) \cite{lei2017deriving} introduces the Weisfeiler-Lehman (WL) graph kernel \cite{shervashidze2011weisfeiler} into the graph neural network for graph classification. The key idea of the WL graph kernel is the classic WL isomorphism test \cite{weisfeiler1968reduction} for the labelled graphs, which iteratively encodes the node label and the sorted node labels of all its neighboring nodes as a new label, as shown in Figure \ref{fig:wl}. Such a relabelling process is repeated $T$ times, and each node is assigned with a new label in the $t$-th iteration such that the $t$-order information is preserved.

Interestingly, the WLKNN implies node representation by passing information from one node to another along the edges and aggregating the node labels in the way of node relabelling. The node representation in the $t$-th iteration can be written as follows
\begin{equation}
  \mathbf{x}_{v}^{(t)}=\sigma\Big(\big(\mathbf{U}_{1}^{(t)}\mathbf{x}_{v}^{(t-1)}\big)\circ\big(\mathbf{U}_{2}^{(t)}\sum\limits_{u\in N(v)}\sigma(\mathbf{U}_{3}^{(t)}\mathbf{x}_{u}^{(t-1)})\big)\Big), 
  \label{eq:wl_kernel}
\end{equation}
where $\mathbf{x}_{v}$ denotes the node $v$'s representation, $t$ is the $t$-th iteration, $\mathbf{U}_{1}$, $\mathbf{U}_{2}$ and $\mathbf{U}_{3}$ represent the shared weight matrices, $\circ$ is the combination operation (e.g., addition) on the node itself and all its neighboring nodes, $\sum$ is the aggregation operation (e.g., max pooling) on all the neighboring nodes and $\sigma$ is the nonlinear activation (e.g., ReLU).

\subsection{MinHash}
Given a universal set $U$ and a subset $S\subseteq U$, MinHash is generated as follows: Assuming that a set of $K$ randomized hash functions (or $K$ random permutations), $\{\pi^{(k)}\}_{k=1}^K$, are applied to $U$, the elements in $S$ which have the minimum hash value in each hash function (or which are placed in the first position of each permutation), $\{\arg\min\big(\pi^{(k)}(S)\big)\}_{k=1}^K$, would be the MinHashes of $S$, and $\{\min\big(\pi^{(k)}(S)\big)\}_{k=1}^K$ would be the corresponding hash function values (or the positions after permutations) of the above MinHashes of $S$ under $\{\pi^{(k)}\}_{k=1}^K$~\cite{broder1998min}.

It is easily observed that all the elements in $U$ are sampled equally because all the elements are mapped to the minimum hash value with equal probability. MinHash is an approximate algorithm for computing the Jaccard similarity of two sets. It has been proved that the probability of two sets, $S$ and $T$, to generate the same MinHash value (hash collision) is exactly equal to the Jaccard similarity of the two sets~\cite{broder1998min}:
\begin{equation}
\Pr\Big(\min\big(\pi^{(k)}(S)\big)=\min\big(\pi^{(k)}(T)\big)\Big) = J(S,T)=\dfrac{|S \cap T|}{|S \cup T|}.
\label{eq:minhash}
\end{equation}

To approximate the expected probability, multiple independent random permutations are used to generate MinHash values. The similarity between two sets based on $K$ MinHashes is calculated by
\begin{equation}
  \hat{J}(S,T) = \frac{\sum_{k=1}^K \mathbbm{1}\Big(\min\big(\pi^{(k)}(S)\big)=\min\big(\pi^{(k)}(T)\big)\Big)}{K}, \label{eq:approx_minhash}
\end{equation}
where $\mathbbm{1}(state)=1$, if $state$ is true, and $\mathbbm{1}(state)=0$, otherwise. As $K \rightarrow \infty$, $\hat{J}(S,T) \rightarrow J(S,T)$.

Considering the high complexity of the random permutation, we practically adopt the hash function, $\pi^{(k)}(i)=\mod(a^{(k)}i+b^{(k)}, c^{(k)})$, where $i$ is the index of the element from $U$, $0 <a^{(k)}, b^{(k)}\\< c^{(k)}$ are two random integers and $c^{(k)}$ is a big prime number such that $c^{(k)}\ge |U|$ \cite{broder1998min}.

\section{Hashing-Accelerated Graph Neural Networks}
\label{sec:algo}

In this section, we will present an efficient hashing model based on the GNN framework, which embeds each node in the attributed network into a low-dimensional space and promotes the similarity-based link prediction. 

\subsection{The Algorithm}

\begin{algorithm}[t]
\fontsize{9pt}{\baselineskip}\selectfont{\caption{The \#GNN Algorithm}
\label{alg:gnn-hash}
\begin{algorithmic}[1]
  \REQUIRE $G=(V,E, f)$, $T$, $K$, $\{\pi^{(t,k)}_{1}, \pi^{(t,k)}_{2}, \pi_{3}^{(t,k)}\}_{t=1,k=1}^{T,K}$ 
  \ENSURE $\mathbf{H}$
  \FOR {$t=1,\cdots,T$}
    \FOR{$k=1,\cdots,K$}
        \FOR {$v\in V$}
            \STATE $x^{(t,k)}_{v,1} \leftarrow \arg\min\big(\pi_{3}^{(t,k)}(\x^{(t-1)}_v)\big)$
        \ENDFOR
        \FOR {$v\in V$}
            \STATE $\x^{(t,k)}_{v,neighbors} \leftarrow \bigcup\limits_{\mathclap{u\in N(v)}}\{x^{(t,k)}_{u,1}\}$ 
            \STATE $\x^{(t)}_v[k] \leftarrow \arg\min\big(\pi^{(t,k)}_{1}(\x^{(t-1)}_v)\bigcup \pi^{(t,k)}_{2}(\x^{(t,k)}_{v,neighbors})\big)$  
        \ENDFOR
    \ENDFOR
    \FOR{$v\in V$}
        \STATE $\mathbf{H}[v,:] \leftarrow \x^{(t)}_{v}\top$
    \ENDFOR
  \ENDFOR
  \RETURN $\mathbf{H}$            
\end{algorithmic}}
\end{algorithm}

We show the proposed model in Algorithm \ref{alg:gnn-hash}. The input contains an attributed network $G=(V,E, f)$, the number of iterations $T$, the size of node representation $K$, and three arrays of randomized hash functions\footnote{In this work, we practically adopt the randomized hash functions to implement the MinHash scheme, and thus each element returns a hash value.} $\{\pi^{(t,k)}_{1}, \pi^{(t,k)}_{2}, \pi^{(t,k)}_{3}\}_{t=1,k=1}^{T,K}$ at the $t$-th iteration and the $k$-th hashing process. The output is $|V|\times K$ matrix $\mathbf{H}$, where each row denotes a node representation.

The proposed algorithm captures high-order node proximity in the Message Passing scheme in a similar way of the WL isomorphism test. In each iteration, Algorithm \ref{alg:gnn-hash} generates the $K$-dimensional representation for each node via $K$ independent MinHash processes. Specifically, in each hashing process\footnote{In the following, we omit $k$ and $t$ in $\pi, x, \x$ for simplicity below.}, it consists of two parts. In Part \uppercase\expandafter{\romannumeral 1} (Lines 3-5), it adopts a MinHash scheme $\pi_{3}$ to summarize the whole network and to allocate each node $v$ message $x_{v, 1}$, which denotes the content diffused from the node itself to all its neighboring nodes. In Part \uppercase\expandafter{\romannumeral 2} (Lines 6-9),  for each node $v$, it first aggregates the messages from all its neighboring nodes produced in Part \uppercase\expandafter{\romannumeral 1}; subsequently, it sketches the node itself $\x_v$ and the corresponding aggregated messages $\x_{v,neighbors}$ by two different MinHash 
schemes $\pi_{1}$ and $\pi_{2}$, and assigns the message which has the minimum hash value under the two MinHash schemes to the corresponding dimension of the node embedding.
After $K$ MinHash operations are finished, the resulting node representation of the attributed network is stored in $\mathbf{H}$ (Lines 11-13). Based on the updated representation, the algorithm starts the next iteration in order to capture higher-order node proximity.

Algorithm \ref{alg:gnn-hash} resembles Eq.~(\ref{eq:wl_kernel}), where random permutation $\pi$ in MinHash is mathematically equivalent with the uniformly distributed random matrix (i.e., random permutation matrix $\Pi\in\{0,1\}^{|A|\times|A|}$). The operations $\bigcup$ (Algorithm \ref{alg:gnn-hash} Line 7), $\bigcup$ (Algorithm \ref{alg:gnn-hash} Line 8) and $\arg\min$ (Algorithm \ref{alg:gnn-hash} Line 8) play the same roles as $\sum$, $\circ$ and $\sigma$ in Eq.~(\ref{eq:wl_kernel}), respectively. It remarkably improves efficiency by avoiding tedious parameter learning.

\subsection{A Markov Chain View}


In Algorithm \ref{alg:gnn-hash}, each node updates the representation by passing and aggregating messages in the GNN framework. From the global perspective, after all the nodes are updated, the whole attributed network transits from one state to another one and more importantly, the whole network state just depends on the last state. 

Based on the above observation, Algorithm \ref{alg:gnn-hash} actually generates a Markov chain composed of the whole attributed network states, each of which implies node representation in the corresponding iteration. Formally, we let $\mathbf{V}^{(t)}=[\mathbf{x}_{v_1}^{(t)}, \mathbf{x}_{v_2}^{(t)}, \cdots, \mathbf{x}_{v_{n}}^{(t)}]\top$ represent the attributed network state at the $t$-th iteration\footnote{We assume $|V|=n$.}, where $\mathbf{V}^{(0)}=[\mathbf{x}_{v_1}^{(0)}, \mathbf{x}_{v_2}^{(0)}, \cdots, \mathbf{x}_{v_{n}}^{(0)}]\top=[\mathbf{v}_{v_1}, \mathbf{v}_{v_2}, \cdots, \mathbf{v}_{v_{n}}]\top$ is the initial attributed network state\footnote{The MinHash functions are unrelated with dimensions of $\x_v$ and $\mathbf{v}_v$, so we unify $\x^{(t)}$ where $t>0$ and $\x^{(0)}$ into $\x^{(t)}$ for convenience.}. Consequently, we have 
\begin{equation}
    \Pr(\mathbf{V}^{(t)}|\mathbf{V}^{(t-1)}, \mathbf{V}^{(t-2)},\cdots, \mathbf{V}^{(0)}) = \Pr(\mathbf{V}^{(t)}|\mathbf{V}^{(t-1)}).
\end{equation}

Considering the fact that each node just captures the information from itself and all its neighboring nodes at the $(t-1)$-th iteration, we define the state transition function $\hbar^{(t)}: \{\mathbf{V}^{(t-1)}\}\mapsto \{\mathbf{V}^{(t)}\}$:

\begin{equation}
    \hbar^{(t)}(\mathbf{V}^{(t-1)})  = 
\left[
\begin{matrix}
\hbar^{(t)}\big((\mathbf{x}_{v_{1}}^{(t-1)}, \bigcup\limits_{\mathclap{u_{1}\in N(v_{1})}}\mathbf{x}_{u_{1}}^{(t-1)})\big) \\
\cdots \\
\hbar^{(t)}\big((\mathbf{x}_{v_{n}}^{(t-1)}, \bigcup\limits_{\mathclap{u_{n}\in N(v_{n})}}\mathbf{x}_{u_{n}}^{(t-1)})\big) \\
\end{matrix}
\right]\\
 = 
\left[
\begin{matrix}
\mathbf{x}_{v_{1}}^{(t)} \\
\cdots \\
\mathbf{x}_{v_{n}}^{(t)} \\
\end{matrix}
\right]=\mathbf{V}^{(t)}, 
\label{eq:transition2}
\end{equation}
where $\hbar^{(t)}$ inherits Eq.~(\ref{eq:wl_kernel}) and maps the node and all its neighboring nodes at the $(t-1)$-th iteration into the embedding at the $t$-th iteration. Evidently, the state transition function is the node-wise operation because it independently feeds one node and its neighboring nodes, yielding the corresponding node representation.

Due to the property of the Markov chain that the whole network state at the $t$-th iteration just relies on the one at the $(t-1)$-th iteration, the proposed \#GNN algorithm is scalable w.r.t. the number of iterations and efficiently captures high-order attribute and structure information based on the last node representation.

\subsection{Theoretical Analysis}
\label{subsec:theo}


\subsubsection{Similarity Estimation}

Given two nodes $v_1$ and $v_2$, their similarity at the $t$-th iteration is
\begin{equation}
    Sim^{(t)}(v_1, v_2) = \mathbb{E}_{\hbar^{(1)}, \cdots, \hbar^{(t)}}[H(\x_{v_1}^{(t)}, \x_{v_2}^{(t)})], 
\label{eq:similarity} 
\end{equation}
where $\x_{v_1}^{(t)}$ and $\x_{v_2}^{(t)}$ are the node representation of $v_1$ and $v_2$ at the $t$-th iteration, respectively, and $H(\cdot,\cdot)$ is the Hamming similarity, i.e., $H(\x_{v_1}^{(t)}, \x_{v_2}^{(t)})=\frac{1}{K}\sum_{i=1}^K\mathbbm{1}(x_{v_1,i}^{(t)}=x_{v_2,i}^{(t)})$, and $x_{v_1,i}^{(t)}$ is the value at the $i$-th dimension of $\x_{v_1}^{(t)}$.

\subsubsection{Representational Power}

The WL isomorphism test is maximally powerful to distinguish the substructures in the graph because the test can map different subtrees composed of the rooted node and its neighboring nodes to different feature vectors \cite{xu2018how,morris2019weisfeiler}. From the perspective of MinHash Eq.~(\ref{eq:minhash}), the maximally powerful GNN algorithm has  
\begin{eqnarray}
   \Pr\Big(\phi^{(t)}\big((\mathbf{x}_{v_{1}}^{(t-1)}, \sum\limits_{\mathclap{u_{1}\in N(v_{1})}}\mathbf{x}_{u_{1}}^{(t-1)})\big) = \phi^{(t)}\big((\mathbf{x}_{v_{2}}^{(t-1)}, \sum\limits_{\mathclap{u_{2}\in N(v_{2})}}\mathbf{x}_{u_{2}}^{(t-1)})\big)\Big) \nonumber\\ =  \left\{
\begin{array}{rcl}
0,    &      & (\mathbf{x}_{v_{1}}^{(t-1)}, \sum\limits_{\mathclap{u_{1}\in N(v_{1})}}\mathbf{x}_{u_{1}}^{(t-1)}) \neq (\mathbf{x}_{v_{2}}^{(t-1)}, \sum\limits_{\mathclap{u_{2}\in N(v_{2})}}\mathbf{x}_{u_{2}}^{(t-1)})\\
1,    &      & (\mathbf{x}_{v_{1}}^{(t-1)}, \sum\limits_{\mathclap{u_{1}\in N(v_{1})}}\mathbf{x}_{u_{1}}^{(t-1)}) = (\mathbf{x}_{v_{2}}^{(t-1)}, \sum\limits_{\mathclap{u_{2}\in N(v_{2})}}\mathbf{x}_{u_{2}}^{(t-1)})
\end{array} \right.
\label{eq:wl}
\end{eqnarray}
where $\sum$ is the aggregation operation on the neighboring nodes, and $\phi$ maps the substructure composed of the node and all its neighboring nodes to the node representation in the embedding space.

By contrast, the GNN algorithms \cite{hamilton2017inductive,kipf2016semi} would be less powerful if they might map different substructures into the same location in the embedding space, 
\begin{equation}
\Pr\Big(\phi^{(t)}\big((\mathbf{x}_{v_{1}}^{(t-1)}, \sum\limits_{\mathclap{u_{1}\in N(v_{1})}}\mathbf{x}_{u_{1}}^{(t-1)})\big) = \phi^{(t)}\big((\mathbf{x}_{v_{2}}^{(t-1)}, \sum\limits_{\mathclap{u_{2}\in N(v_{2})}}\mathbf{x}_{u_{2}}^{(t-1)})\big)\Big) >0. 
\label{eq:gnn}
\end{equation}

The \#GNN algorithm could map two substructures in the graph into the same location with the probability of their similarity,
\begin{equation}
\Pr\Big(\hbar^{(t)}\big((\mathbf{x}_{v_{1}}^{(t-1)}, \bigcup\limits_{\mathclap{u_{1}\in N(v_{1})}}\mathbf{x}_{u_{1}}^{(t-1)})\big) = \hbar^{(t)}\big((\mathbf{x}_{v_{2}}^{(t-1)}, \bigcup\limits_{\mathclap{u_{2}\in N(v_{2})}}\mathbf{x}_{u_{2}}^{(t-1)})\big)\Big) = Sim^{(t)}(v_1, v_2). 
\label{eq:gnn}
\end{equation}


\subsubsection{Concentration}

Based on the Markov chain composed of a sequence of the attributed network states derived from Algorithm \ref{alg:gnn-hash}, i.e., $\mathbf{V}=(\mathbf{V}^{(0)}, \mathbf{V}^{(1)}, \cdots, \mathbf{V}^{(t)})$, and $\{\mathbf{V}\}$ is a set of all possible values of $\mathbf{V}$, we can show a highly-concentrated estimator of $Sim^{(t)}(v_1, v_2)$.


\begin{lemma}
\label{lemma:lipschitz}
Given a Markov chain derived from Algorithm \ref{alg:gnn-hash}, $\mathbf{V}=(\mathbf{V}^{(0)}, \mathbf{V}^{(1)}, \cdots, \mathbf{V}^{(t)})$, and a function $\varphi(\mathbf{V})=H(\x_{v_1}^{(t)},\x_{v_2}^{(t)})$, where $H(\x_{v_1}^{(t)},\x_{v_2}^{(t)})\in [0,1]$ is the Hamming similarity of any two nodes $v_1$ and $v_2$ in $\mathbf{V}^{(t)}$, it concludes that there definitely exists some $c$ such that $\varphi$ is the $c$-Lipschitz function w.r.t the normalized Hamming distance.
\end{lemma}

\begin{proof}
Considering any two different instances of the Markov chain, $\mathbf{V}_{1}=(\mathbf{V}_{1}^{(0)}, \mathbf{V}_{1}^{(1)}, \cdots, \mathbf{V}_{1}^{(t)}), \mathbf{V}_2 =(\mathbf{V}_{2}^{(0)}, \mathbf{V}_{2}^{(1)}, \cdots, \mathbf{V}_{2}^{(t)}) \in \{\mathbf{V}\}$, we have $\varphi(\mathbf{V}_{1})=H_{1}(\x_{v_1}^{(t)}, \x_{v_2}^{(t)})$ and $\varphi(\mathbf{V}_{2}) =H_{2}(\x_{v_1}^{(t)}, \x_{v_2}^{(t)})$, where $H_{1}(\x_{v_1}^{(t)}, \x_{v_2}^{(t)}), H_{2}(\x_{v_1}^{(t)}, \x_{v_2}^{(t)})\in [0,1]$ are the Hamming similarities of any two nodes $v_1$ and $v_2$ in $\mathbf{V}_{1}^{(t)}$ and $\mathbf{V}_2^{(t)}$, respectively. The normalized Hamming distance between $\mathbf{V}_{1}$ and $\mathbf{V}_2$ is $d(\mathbf{V}_{1}, \mathbf{V}_2) = \frac{1}{t+1}\sum\nolimits_{i=0}^{t}\mathbbm{1}(\mathbf{V}_{1}^{(i)}\neq \mathbf{V}_{2}^{(i)})$. First, we have $|\varphi(\mathbf{V}_{1}) -\varphi(\mathbf{V}_{2})| = |H_{1}(\x_{v_1}^{(t)},\\ \x_{v_2}^{(t)}) -H_{2}(\x_{v_1}^{(t)}, \x_{v_2}^{(t)})|\leq 1$ due to $H_{1}(\x_{v_1}^{(t)}, \x_{v_2}^{(t)}), H_{2}(\x_{v_1}^{(t)}, \x_{v_2}^{(t)})\in[0,1]$. Second, let $|\mathbf{V}_{1}-\mathbf{V}_{2}|$ denote the normalized Hamming distance, and we have $|\mathbf{V}_{1}-\mathbf{V}_{2}|=d(\mathbf{V}_{1}, \mathbf{V}_{2})\leq 1$. Naturally, there is definitely some $c$ (e.g., $c=t+1$) such that $|\varphi(\mathbf{V}_{1})-\varphi(\mathbf{V}_{2})|\leq c|\mathbf{V}_{1}-\mathbf{V}_{2}|$ because of  $\sum\nolimits_{i=0}^{t}\mathbbm{1}(\mathbf{V}_{1}^{(i)}\neq \mathbf{V}_{2}^{(i)})\geq 1$ when $\mathbf{V}_{1}\neq \mathbf{V}_{2}$. $|\varphi(\mathbf{V}_{1})-\varphi(\mathbf{V}_{2})|\leq c|\mathbf{V}_{1}-\mathbf{V}_{2}|$ evidently holds when $\mathbf{V}_{1}=\mathbf{V}_{2}$. The conclusion has been drawn.
\end{proof}

\begin{theorem}
Given a sequence of Hamming similarities between any two nodes $v_1$ and $v_2$, which are derived from Algorithm \ref{alg:gnn-hash}, $H_{v_1,v_2}=(H(\x_{v_1}^{(0)}, \x_{v_2}^{(0)}), H(\x_{v_1}^{(1)}, \x_{v_2}^{(1)}), \cdots, H(\x_{v_1}^{(t)}, \x_{v_2}^{(t)}))$. Then for some $c>0$, a specified value $M_t$ and any $\epsilon>0$, we have
    \begin{equation}
        \Pr[|H(\x_{v_1}^{(t)}, \x_{v_2}^{(t)}) - Sim^{(t)}(v_1, v_2)|\geq \epsilon] \leq 2\exp (-\frac{t\epsilon^2}{2c^2M_t^2})\nonumber
    \end{equation}
    \label{theo:bound}
\end{theorem}
\begin{proof}
Considering a Markov chain yielded by Algorithm \ref{alg:gnn-hash}, $\mathbf{V}=(\mathbf{V}^{(0)}, \mathbf{V}^{(1)}, \cdots, \mathbf{V}^{(t)})$, and a function $\varphi(\mathbf{V})=H(\x_{v_1}^{(t)},\x_{v_2}^{(t)})$, we can verify from Lemma~\ref{lemma:lipschitz} that $\varphi(\mathbf{V})=H(\x_{v_1}^{(t)},\x_{v_2}^{(t)})$ is $c$-Lipschitz w.r.t. the normalized Hamming distance. $M_t$ is a specific value following Eqs~(1.6-1.8) in \cite{kontorovich2008concentration}. Finally, from Theorem 1.2 in \cite{kontorovich2008concentration} we have 
  \begin{align*}
        & \Pr[|H(\x_{v_1}^{(t)}, \x_{v_2}^{(t)}) - Sim^{(t)}(v_1, v_2)|\geq \epsilon]  \\
   =    & \Pr[|H(\x_{v_1}^{(t)}, \x_{v_2}^{(t)}) - \mathbb{E}[H(\x_{v_1}^{(t)}, \x_{v_2}^{(t)})]|\geq \epsilon]  \\
   =    & \Pr[|\varphi(\mathbf{V}) - \mathbb{E}[\varphi(\mathbf{V})]|\geq \epsilon] \\
  \leq  & 2\exp (-\frac{t\epsilon^2}{2c^2M_t^2}).
  \end{align*}
  The conclusion has been drawn.
\end{proof}

\subsubsection{Complexity:} Let $\overline{\nu}$ be the average degree of the network, $T$ be the number of iterations and $K$ be the size of node embeddings of Algorithm \ref{alg:gnn-hash}. The first inner for loop (Lines 3-5) costs $\mathcal{O}(|V|K)$\footnote{Initially, the nodes are represented as the feature vectors with the length of $|A|$, but in the following, the size of the vectors is $K$.}, and the second one (Lines 6-9) spends $\mathcal{O}(|V|(K+\overline{\nu}))$, so the overall time complexity is $\mathcal{O}(TK|V|(K+\overline{\nu}))$. Obviously, \#GNN runs linearly w.r.t. $T$ and $|V|$. In terms of space complexity, Algorithm \ref{alg:gnn-hash} maintains each node embedding at each iteration, i.e., $\mathcal{O}(|V|K)$ (Line 8), and requires spaces of $\mathcal{O}(|V|)$ to summarize the whole network (Lines 3-5). Therefore, the overall space complexity is $\mathcal{O}(|V|K)$.

\subsection{Discussion}

The one that is the most relevant to the proposed \#GNN algorithm is NetHash because they both exploit MinHash to hierarchically sketch the node itself and its high-order neighboring nodes in the attributed network, and their differences lie in 
\begin{itemize}
    \item \#GNN sketches from the node itself to the predefined highest-order neighboring nodes, while NetHash does inversely.
    \item \#GNN updates all node representation in each iteration and makes them shared in the whole network at the next iteration. NetHash has to independently sketch each rooted tree from leaves to the root node in order to generate all node representation. 
    \item \#GNN produces a Markov chain composed of a sequence of attributed network states, which keeps the time complexity growing linearly w.r.t. the number of iterations and shows good scalability when exploring higher-order node proximity. The independent tree sketching adopted by NetHash makes the time complexity exponential with the depth of the rooted trees increasing, and only limited-order node proximity could be captured practically.
\end{itemize}

\section{Experimental Results}
\label{sec:exp}

In this section, we conduct extensive experiments in order to evaluate the performance of the proposed \#GNN algorithm.

\subsection{Datasets}

We conduct the link prediction task on five attributed network data. 

\begin{enumerate}
    \item Twitter~\cite{shen2018flexible}: This dataset consists of 2,511 Twitter users as nodes and 37,154 following relationships as edges. The hashtags specified by users act as attribute information. The number of the attributes is 9,073 and the average degree of the network is 29.59.
    \item Facebook~\cite{leskovec2012learning}: This is an ego-network which was collected from survey participants using the Facebook app. There are 4,039 nodes and 88,234 links in the network. Each node as a Facebook user is described by a 1,403-dimensional feature vector. The average degree of the network is 43.69.
    \item BlogCatalog~\cite{li2015unsupervised}: This dataset from an image and video sharing website consists of 5,196 users as nodes and 171,743 following relationship as edges in the network. The tags of interest specified by users act as attribute information. The number of the attributes is 8,189 and the average degree of the network is 66.11.
    \item Flickr~\cite{li2015unsupervised}: This dataset consists of 7,564 bloggers as nodes and 239,365 following relationship as edges in the network. The keywords in the blogers' blog descriptions are attribute information. The number of the attributes is 12,047 and the average degree of the network is 63.29.
    \item Google+~\cite{shen2018flexible}: This is an ego-network of Google+ users and and the friendship relationship. The network consists of 7,856 nodes and 321,268 edges, and each node is modelled as a 2,024-dimensional feature vector from user profiles. The average degree of the network is 81.79.
\end{enumerate}
The above datasets are summarized in Table \ref{tab:data}.

\begin{table}[t]
\setlength{\abovecaptionskip}{5pt}%
\setlength{\belowcaptionskip}{0pt}%
  \centering
  \caption{Summary of the network datasets.} \label{tab:data}
  \begin{normalsize}
  \fontsize{8pt}{\baselineskip}\selectfont{\begin{tabular}{lrrrr}
  \hline\noalign{\smallskip}
  Data Set     & $|V|$ & $|E|$  & $|A|$ & $\overline{\nu}$\\
  \noalign{\smallskip}\hline\noalign{\smallskip}
  Twitter   & 2,511 &  37,154   &   9,073   & 29.59 \\  
  Facebook  & 4,039 &  88,234   &  1,403        & 43.69 \\
  BlogCatalog & 5,196 & 171,743 & 8,189 & 66.11 \\
  Flickr    & 7,564 & 239,365   & 12,047        & 63.29      \\
  Google+  &  7,856  & 321,268  & 2,024 & 81.79  \\
  \noalign{\smallskip}\hline
  \end{tabular}
  }
  \end{normalsize}
  \begin{flushleft}
    \normalsize{$|V|$: number of nodes, $|E|$: number of edges, $\overline{\nu}$: average value of node degrees, $|A|$: size of attribute set.}
  \end{flushleft}
\end{table}

\subsection{Experimental Preliminaries}
\label{subsec:pre}

Five state-of-the-art link prediction methods, which can preserve both attribute and structure information in the attributed network, are compared in our experiments. In order to evaluate the peak performance of all the algorithms and thus to make our comparisons fair, we configure the key parameter settings of each algorithm on each data via grid search in our experiments. For the similarity-based algorithms (i.e., the hashing-based algorithms), the number of dimensions for node representation, $K$, is set to be 200, which is a common practice used in attributed network embedding \cite{yang2015network,tu2017cane,wu2018efficient}. 
Additionally, all the hashing-based algorithms and SEAL employ a parameter to control the order of the neighboring nodes which are captured by the algorithms, and we denote $T$ as the corresponding parameter for the above algorithms\footnote{We would like to note that, in the original literature of the algorithms, the parameter to control the order of the neighboring nodes is denoted as different notations. In this work, in order to avoid ambiguity, such a parameter is denoted as $T$.}. For SEAL and P-GNN, we adopt the hyper-parameter settings, which are recommended by their authors. Particularly, if the algorithms cannot converge within the recommended number of the epochs, we would employ a larger value, i.e., 10,000. We would like to note that those algorithms with the larger epoch numbers might not converge as well. If this is the case, we will report the result of the corresponding algorithms with the recommended hyper-parameter setting of epoch number. 


\begin{enumerate}
    
    \item \textbf{BANE}~\cite{yang2018binarized}: It learns binary hash codes for the attributed network by adding the binary node representation constraint on the Weisfeiler-Lehman matrix factorization. Following \cite{yang2018binarized}, we set the regularization parameter $\alpha$ to the recommended value 0.001, and then test $T\in\{1,2,3,4,5\}, \gamma\in\{0.1, 0.2, \cdots, 0.8, 0.9\}$, and obtain $T=5, \gamma=0.4$ on Twitter, $T=4, \gamma=0.5$ on Facebook, $T=2, \gamma=0.9$ on BlogCatalog, $T=3, \gamma=0.8$ on Flickr and $T=5, \gamma=0.6$ on Google+. 
    
    \item \textbf{LQANR}~\cite{yang2019low}: It learns low-bit hash codes and the layer aggregation weights under the low-bit quantization constraint based on matrix factorization. Following \cite{yang2019low}, we set the regularization parameter $\beta$ to the recommended value 0.001, and then test $T\in\{1,2,3,4,5\}, r\in\{1,1.1,1.2,\cdots,9.9, 10\}$, and acquire $T=5, r=10$ on Twitter, $T=4, r=2.7$ on Facebook, $T=1, r=1.7$ on BlogCatalog, $T=1, r=9.7$ on Flickr and $T=4, r=8.9$ on Google+.

    \item \textbf{NetHash}~\cite{wu2018efficient}: It encodes attribute and structure information of each node by exploiting the randomized hashing technique to recursively sketch the tree rooted at the node. We test $T\in\{1,2,3,4,5\}$, and have $T=2$ on Twitter, $T=1$ on Facebook, BlogCatalog, Flickr and Google+. Besides, their entropies are 4.32, 4.69, 4.89, 4.81 and 5.30, respectively. NetHash runs out of memory on Twitter and Facebook when $T>4$, and on BlogCatalog, Flickr and Google+ when $T>3$.
    

    
    \item \textbf{SEAL}~\cite{zhang2018link}: It is a GNN-based link prediction method which captures both attribute and structure information by extracting the variable-sized subgraphs composed of $T$-hop neighboring nodes. Following the suggestions of the author\footnote{\url{https://github.com/muhanzhang/SEAL/tree/master/Python}}, we tune \texttt{-{}-max-train-num} on $\{1,000, 10,000, 100,000\}$ and \texttt{-{}-max-\\nodes-per-hop} on $\{10, 100, None\}$, and finally set \texttt{-{}-max-train-\\num} to $100,000$ and \texttt{-{}-max-nodes-per-hop} to $None$ (i.e., the recommended value) on Twitter and Facebook; \texttt{-{}-max-train-num} to $10,000$ and \texttt{-{}-max-nodes-per-hop} $10$ on BlogCatalog and Google+; \texttt{-{}-max-train-num} to $1,000$ and \texttt{-{}-max-nodes-per-hop} $10$ on Flickr, because the algorithm would run out of memory when picking up the larger values. We test $T\in\{1,2,3,4,5\}$, and have $T=1$ on Twitter, Facebook, BlogCatalog and Flickr, and $T=2$ on Google+. SEAL runs out of memory on Twitter and Facebook when $T>1$, and on BlogCatalog and Flickr when $T>2$. 
    
    \item \textbf{P-GNN}~\cite{you2019position}: The GNN-based algorithm preserves attribute and structure information and further improves performance by incorporating node position information w.r.t. all other nodes in the network. We test the number of layers $L\in\{1,2\}$ and the number of the hops of the shortest path distance $k\in\{-1, 2\}$ where -1 means the exact shortest path. Finally, we have $L=2$ and $k=-1$. 
    
    \item \textbf{\#GNN}: It is our proposed algorithm. We test $T\in\{1,2,3,4,5\}$, and have $T=3$ on Twitter, $T=3$ on Facebook, $T=4$ on BlogCatalog, $T=1$ on Flickr and $T=5$ on Google+.
\end{enumerate}

The executable programs of all the competitors are kindly provided by their authors. The hashing-based algorithms return each node representation by embedding nodes into $l_1$ (Hamming) space. Suppose the hashing-based algorithms generate $\x_{v_1}$ and $\x_{v_2}$, which are the embeddings with the length of $K$, for two nodes $v_1$ and $v_2$, respectively, the similarity between $v_1$ and $v_2$ is defined as
\begin{equation}
    Sim(\x_{v_1},\x_{v_2}) = \frac{\sum_{k=1}^{K}\mathbbm{1}(x_{v_1,k} = x_{v_2,k})}{K},
    \label{eq:sim}
\end{equation}
where $\mathbbm{1}(state) = 1$ if $state$ is true, and $\mathbbm{1}(state) = 0$ otherwise. All the experiments are conducted on Linux with NVIDIA Tesla GPU (11GB RAM), 2.60 GHz Intel 12-Core Xeon CPUs and 220GB RAM. 


We have released the datasets and the source code of the \#GNN algorithm in \url{https://github.com/williamweiwu/williamweiwu.github.io/tree/master/Graph_Network\%20Embedding/HashGNN}.

\subsection{Link Prediction}
\label{subsec:prediction}

\begin{table*}[t]
\setlength{\abovecaptionskip}{5pt}%
\setlength{\belowcaptionskip}{0pt}%
  \centering
  \caption{Link prediction performance results} \label{tab:prediction}
  \begin{normalsize}
  \fontsize{8pt}{\baselineskip}\selectfont{\begin{tabular}{llrrrrrrrrrr}
  \hline\noalign{\smallskip}
   \multirow{4}*{\textbf{Data}} & \multirow{4}*{\textbf{Algorithms}} &  \multicolumn{10}{c}{\textbf{Training Ratio}} \\ \cmidrule(lr){3-12}
  \noalign{\smallskip}
  \multirow{3}*{} &\multirow{3}*{} &  \multicolumn{2}{c}{50\%} & \multicolumn{2}{c}{60\%} & \multicolumn{2}{c}{70\%} & \multicolumn{2}{c}{80\%} & \multicolumn{2}{c}{90\%}\\  \cmidrule(lr){3-4} \cmidrule(lr){5-6} \cmidrule(lr){7-8} \cmidrule(lr){9-10} \cmidrule(lr){11-12}
  \multirow{3}*{} &\multirow{3}*{} &  \tabincell{c}{AUC (\%)}  &  \tabincell{c}{Runtime (s)}   & \tabincell{c}{AUC (\%)}   &  \tabincell{c}{Runtime (s)}  & \tabincell{c}{AUC (\%)}  & \tabincell{c}{Runtime (s)} & \tabincell{c}{AUC (\%)}   &  \tabincell{c}{Runtime (s)}     & \tabincell{c}{AUC (\%)} &  \tabincell{c}{Runtime (s)}  \\
  \noalign{\smallskip}\hline\noalign{\smallskip}
  \multirow{6}*{Twitter} & BANE & \textbf{98.17} & 90.28 & \textbf{98.31} & 97.92 & \textbf{98.46} & 106.39 & 98.53 & 112.84 & 98.67 & 116.61 \\ 
& LQANR & 96.85 & 38.25 & 97.53 & 39.72 & 97.67 & 40.41 & 97.89 & 43.63 & 97.98 & 44.35 \\ 
& NetHash & 91.10 & 1.52 & 90.95 & 2.08 & 90.69 & 2.71 & 91.28 & 3.48 & 91.54 & 4.37 \\ 
& SEAL & 95.82 & 3229.18 & 96.18 & 3609.05 & 96.44 & 3998.78 & 96.60 & 4346.25 & 96.82 & 4701.18 \\ 
& P-GNN & 95.64 & 951.30 & 95.73 & 956.47 & 95.14 & 964.40 & 96.13 & 973.98 & 96.33 & 980.70 \\ 
& \#GNN & 97.96 & 2.64 & 98.24 & 2.87 & 98.41 & 2.99 & \textbf{98.57} & 3.07 & \textbf{98.82} & 3.25 \\ 
  \noalign{\smallskip}\hline\noalign{\smallskip}
  \multirow{6}*{Facebook}& BANE & \textbf{98.10} & 64.08 & \textbf{98.22} & 69.19 & \textbf{98.31} & 70.06 & \textbf{98.26} & 72.39 & 98.29 & 76.65 \\ 
& LQANR & 97.27 & 31.71 & 97.24 & 33.38 & 97.33 & 34.76 & 97.30 & 35.28 & 97.29 & 35.72 \\ 
& NetHash & 94.41 & 0.35 & 95.17 & 0.40 & 96.52 & 0.45 & 97.02 & 0.49 & 97.40 & 0.54 \\ 
& SEAL & 95.54 & 2457.20 & 95.83 & 2604.39 & 95.92 & 2740.08 & 96.07 & 2857.41 & 96.21 & 2969.81 \\ 
& P-GNN & 94.89 & 1197.08 & 93.73 & 1214.99 & 94.90 & 1230.97 & 94.54 & 1244.26 & 94.69 & 1253.85 \\ 
& \#GNN & 97.93 & 3.99 & 98.18 & 4.05 & 98.23 & 4.25 & \textbf{98.26} & 4.45 & \textbf{98.42} & 4.54 \\ 
  \noalign{\smallskip}\hline\noalign{\smallskip}
  \multirow{6}*{\tabincell{c}{Blog\\-Catalog}} 
                        & BANE & \textbf{75.95} & 254.59 & 76.63 & 262.27 & 76.71 & 273.54 & 77.42 & 280.61 & \textbf{77.26} & 296.34 \\ 
& LQANR & 65.54 & 64.60 & 65.30 & 66.52 & 64.53 & 67.76 & 64.72 & 71.45 & 64.34 & 73.06 \\ 
& NetHash & 68.82 & 1.44 & 71.04 & 1.63 & 71.38 & 1.83 & 71.30 & 2.04 & 73.21 & 2.25 \\ 
& SEAL & 61.95 & 2255.61 & 62.35 & 2500.85 & 62.46 & 2864.73 & 62.84 & 3219.18 & 63.15 & 3553.29 \\ 
& P-GNN & 71.95 & 11005.72 & \textbf{77.52} & 11928.24 & \textbf{77.06} & 13095.90 & \textbf{77.82} & 14585.58 & 75.23 & 19210.06 \\ 
& \#GNN & 70.10 & 9.16 & 71.52 & 9.24 & 71.87 & 9.46 & 71.53 & 9.69 & 72.96 & 10.15 \\
  \noalign{\smallskip}\hline\noalign{\smallskip}            
  \multirow{6}*{Flickr} & BANE & 77.74 & 342.18 & 78.21 & 374.25 & 79.08 & 394.32 & 79.35 & 419.36 & 79.32 & 428.57 \\ 
& LQANR & 74.29 & 75.03 & 69.55 & 75.67 & 66.05 & 78.83 & 63.98 & 87.64 & 62.27 & 91.78 \\ 
& NetHash & 84.64 & 1.11 & 84.72 & 1.27 & 85.59 & 1.46 & 85.17 & 1.63 & 85.06 & 1.80 \\ 
& SEAL & 59.12 & 557.81 & 60.45 & 1053.82 & 61.33 & 1264.72 & 62.48 & 1887.83 & 62.62 & 2411.04 \\ 
& P-GNN & \textbf{85.59} & 17537.72 & \textbf{87.00} & 18027.87 & \textbf{87.10} & 19200.07 & \textbf{87.19} & 19444.35 & \textbf{86.79} & 19584.91 \\ 
& \#GNN & 85.42 & 1.50 & 85.78 & 1.53 & 86.17 & 1.64 & 86.37 & 1.76 & 86.33 & 1.87 \\ 
  \noalign{\smallskip}\hline\noalign{\smallskip}            
  \multirow{6}*{Google+} & BANE & 83.74 & 190.76 & 83.77 & 236.59 & 84.12 & 240.34 & 83.89 & 259.90 & 84.04 & 276.92 \\ 
& LQANR & 82.91 & 25.10 & 82.66 & 26.26 & 82.42 & 27.54 & 82.77 & 29.10 & 82.69 & 29.28 \\ 
& NetHash & 76.87 & 0.44 & 76.85 & 0.50 & 76.78 & 0.58 & 77.61 & 0.65 & 77.97 & 0.73 \\ 
& SEAL & 85.06 & 2528.79 & 85.99 & 2746.96 & 86.58 & 3723.42 & 87.45 & 4531.50 & 87.97 & 4859.09 \\ 
& P-GNN & 88.97 & 2531.58 & \textbf{90.94} & 2602.67 & 85.18 & 2621.37 & 87.99 & 2987.09 & 86.04 & 3010.92 \\ 
& \#GNN & \textbf{89.80} & 14.97 & 90.03 & 15.89 & \textbf{90.17} & 16.78 & \textbf{89.76} & 17.23 & \textbf{90.74} & 17.84 \\  
  \noalign{\smallskip}\hline
  \end{tabular}
  }
  \end{normalsize}
\end{table*}

In the link prediction task, we randomly preserve a training ratio of edges (i.e., $\{50\%, 60\%, 70\%, 80\%, 90\%\}$) as a training network, and use the deleted edges as testing links. The compared hashing-based algorithms and the \#GNN algorithm acquire the node representation based on the training network and then rank testing links and nonexistent ones in terms of similarity between each pair of nodes computed by Eq.~(\ref{eq:sim}); the GNN-based algorithms conduct the task in an end-to-end manner. We adopt AUC to measure the performance in terms of accuracy, which reflects the probability that a random selected testing link ranks above a random selected nonexistent one. We set a cutoff time of 24 hours for each algorithm, and report the average AUC from 5 repeated trials in each setting.

Table \ref{tab:prediction} reports the experimental results in terms of accuracy and end-to-end runtime. Generally, \#GNN competes well with BANE and LQANR and even defeat them in some cases, and outperforms NetHash in terms of accuracy. This illustrates that the GNN framework enables \#GNN to powerfully propagate and capture information without fitting the specific data distribution; and also it is better than the independent tree sketching.
Similarly, \#GNN achieves the accuracy comparable to SEAL and P-GNN. In some cases, SEAL and P-GNN are even inferior, possibly because they cannot converge given the settings of epochs, even though they use a larger value (i.e., 10,000). Particularly, SEAL performs worst on BlogCatalog and Flickr. This is not surprising because it captures insufficient information in the settings of the smaller \texttt{-{}-max-train-num} and  \texttt{-{}-max-nodes-per-hop}. By contrast, SEAL weakens the above adverse effects by capturing the beneficial 2nd-order neighboring nodes on Google+.
Besides, NetHash and SEAL practically explore limited-order neighboring nodes because rooted trees in NetHash and subgraphs in SEAL expand dramatically when $T$ increases, as mentioned in the parameter settings. It is worth noting that the number of attributes also impacts SEAL --- although Google+ has the most nodes and edges, it has smaller size of the attribute set, and thus SEAL can explore higher-order neighboring nodes; Flickr has the largest number of attributes, so SEAL has to use the smallest \texttt{-{}-max-train-num}.

In terms of runtime, \#GNN performs much faster than the learning-based algorithms and particularly, by $2\sim 4$ orders of magnitudes faster than SEAL and P-GNN, because the randomized hashing approach does not need learning. SEAL runs more quickly on Flickr than other datasets, mainly because \texttt{-{}-max-train-num} is smaller on Flickr (i.e., 1,000) than other datasets (i.e., 10,000 or 100,000). P-GNN performs more slowly on BlogCatalog and Flickr due to a larger epoch number (i.e., 10,000). NetHash generally runs faster than \#GNN, largely because NetHash achieves the best performance when $T\in\{1,2\}$ while \#GNN commonly explores higher-order neighboring nodes. Exceptionally, \#GNN runs more quickly on Twitter than NetHash in the cases of $\{80\%,90\%\}$, although the former captures up to 3-order proximity while the latter does just 2-order proximity, because the exponentially increased time from $T=1$ to $T=2$ in NetHash surpasses the linearly increased time from $T=1$ to $T=3$ in \#GNN in the denser training networks.
Particularly, NetHash performs slightly faster than \#GNN on Flickr because the two algorithms both peaks at $T=1$ and \#GNN spends a little more time in summarizing the network.

\subsection{Embedding Efficiency}
\label{sec:embedding_efficiency}

\begin{table*}[t]
\setlength{\abovecaptionskip}{5pt}%
\setlength{\belowcaptionskip}{0pt}%
  \centering
  \caption{Embedding runtime of the hashing-based algorithms} \label{tab:embedding}
  \begin{normalsize}
  \fontsize{8pt}{\baselineskip}\selectfont{\begin{tabular}{llrrrrr}
  \hline\noalign{\smallskip}
   \multirow{5}*{\textbf{Data}} & \multirow{5}*{\textbf{Algorithms}} &  \multicolumn{5}{c}{\textbf{Training Ratio}} \\ \cmidrule(lr){3-7}
  \noalign{\smallskip}
  \multirow{3}*{} &\multirow{3}*{} &  50\% & 60\% & 70\% & 80\% & 90\%\\  \cmidrule(lr){3-3} \cmidrule(lr){4-4} \cmidrule(lr){5-5} \cmidrule(lr){6-6} \cmidrule(lr){7-7}
  \noalign{\smallskip}
  \multirow{3}*{} &\multirow{3}*{} & \tabincell{c}{Embedding \\ Time (s)}   & \tabincell{c}{Embedding \\ Time (s)}  &  \tabincell{c}{Embedding \\ Time (s)} & \tabincell{c}{Embedding \\ Time (s)}  &  \tabincell{c}{Embedding \\ Time (s)}  \\
  \noalign{\smallskip}\hline\noalign{\smallskip}
  \multirow{4}*{Twitter} & BANE & 90.20 & 97.84 & 106.31 & 112.76 & 116.54 \\ 
& LQANR & 38.17 & 39.64 & 40.34 & 43.55 & 44.28 \\ 
& NetHash & 1.43 & 1.99 & 2.62 & 3.39 & 4.28 \\ 
& \#GNN & 2.55 & 2.79 & 2.90 & 2.98 & 3.17 \\ 
  \noalign{\smallskip}\hline\noalign{\smallskip}
  \multirow{4}*{Facebook} & BANE & 64.00 & 69.11 & 69.98 & 72.31 & 76.57 \\ 
& LQANR & 31.63 & 33.30 & 34.69 & 35.21 & 35.65 \\ 
& NetHash & 0.27 & 0.32 & 0.36 & 0.41 & 0.46 \\ 
& \#GNN & 3.92 & 3.97 & 4.17 & 4.35 & 4.45 \\   \noalign{\smallskip}\hline\noalign{\smallskip}
  \multirow{4}*{\tabincell{c}{Blog\\-Catalog}} 
& BANE & 254.49 & 262.18 & 273.45 & 280.52 & 296.25 \\ 
& LQANR & 64.51 & 66.43 & 67.67 & 71.36 & 72.98 \\ 
& NetHash & 1.34 & 1.54 & 1.73 & 1.95 & 2.16 \\ 
& \#GNN & 9.07 & 9.14 & 9.37 & 9.59 & 10.07 \\  \noalign{\smallskip}\hline\noalign{\smallskip}                  
  \multirow{4}*{Flickr}& BANE & 342.09 & 374.16 & 394.24 & 419.27 & 428.48 \\ 
& LQANR & 74.94 & 75.57 & 78.74 & 87.55 & 91.69 \\ 
& NetHash & 1.02 & 1.18 & 1.37 & 1.54 & 1.71 \\ 
& \#GNN & 1.42 & 1.45 & 1.55 & 1.66 & 1.78 \\  \noalign{\smallskip}\hline\noalign{\smallskip}            
  \multirow{4}*{Google+}& BANE & 190.66 & 236.48 & 240.23 & 259.80 & 276.82 \\ 
& LQANR & 24.99 & 26.16 & 27.44 & 28.99 & 29.18 \\ 
& NetHash & 0.34 & 0.41 & 0.48 & 0.55 & 0.63 \\ 
& \#GNN & 14.86 & 15.79 & 16.68 & 17.14 & 17.74 \\   \noalign{\smallskip}\hline
  \end{tabular}
  }
  \end{normalsize}
\end{table*}

The GNN-based link prediction approaches commonly perform the task in an end-to-end manner, while the traditional methods embed each node into the low-dimensional space and then explicitly compute the similarity between the compact node embeddings. Here, we show the embedding efficiency of the hashing-based algorithms.

Table \ref{tab:embedding} reports the embedding time of the hashing-based algorithms in the cases of the above-mentioned training ratios. We observe from Tables \ref{tab:prediction} and \ref{tab:embedding} that the embedding process dominates link prediction in terms of runtime. Therefore, following Table \ref{tab:prediction}, the randomized hashing algorithms (\#GNN and NetHash) shows more efficient embedding process than the learning-based algorithms (BANE and LQANR) in Table \ref{tab:embedding} because the former avoids the complicated matrix factorization operation in BANE and LQANR.

\subsection{Scalability}
\label{subsec:scalability}

As demonstrated in Section \ref{sec:embedding_efficiency}, the efficiency of the similarity-based link prediction algorithms highly depends on the node embedding, and thus we test scalability of the hashing-based algorithms in the embedding process. To this end, we select a large-scale DBLP attributed network with millions of nodes and furthermore, generate node representations on four subnetworks with increasing sizes from $10^3$ to $10^6$ nodes and constant average degree of 20. Because the parameters except $T$ and $K$ cannot impact the embedding time, we test the algorithms under the parameter configuration and $K=200$. We report the average embedding time w.r.t. the number of nodes in the cases of $T\in\{1,2,3,4,5\}$ and w.r.t. $T$ in the original DBLP network. We set a cutoff time of 24 hours for each algorithm.
\begin{enumerate}
    \item DBLP~\cite{tang2008arnetminer}: The original data is a collection of papers from DBLP, which consists of 4,107,340 papers and 36,624,464 citation relationship. We clean up papers without abstracts or citations and build a citation network with 3,273,363 papers being nodes and 32,533,645 citation being edges, where each node uses abstract as attributes. The number of attributes is 770,785 and the average degree of the network is 19.88.
\end{enumerate}

\begin{figure}[t]
\setlength{\abovecaptionskip}{5pt}%
\setlength{\belowcaptionskip}{0pt}%
\centering
\includegraphics[width=\linewidth]{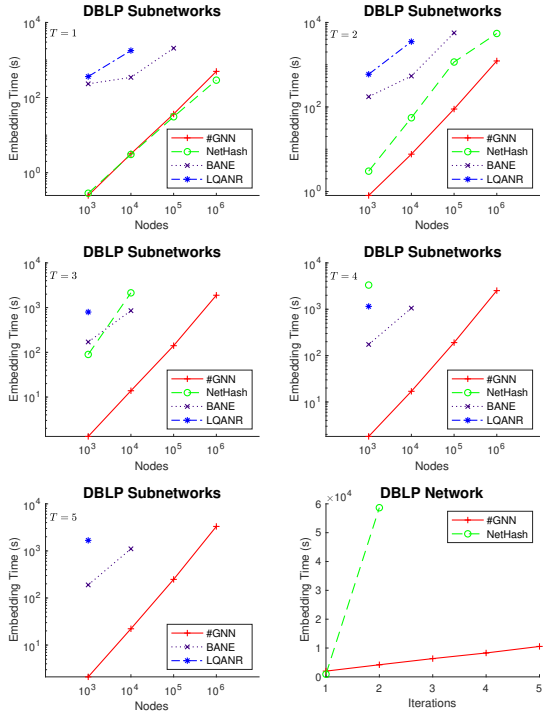}
  \caption{Scalability of the hashing-based algorithms on the DBLP network and its four subnetworks. The first five subplots show scalability w.r.t. the nodes of the subnetworks in the cases of $T\in\{1,2,3,4,5\}$, and the last one shows scalability w.r.t. $T$ on the original network. Note that $x$-axis and $y$-axis are scaled to log in the first five subplots.}
\label{fig:dblp11_scalability}
\end{figure}

In Figure \ref{fig:dblp11_scalability}, we empirically observe that \#GNN scales linearly with the number of nodes in all cases of $T$. \#GNN is able to generate node representation for 1 million nodes in less than 10,000 seconds on the four subnetworks. By contrast, the compared hashing algorithms run out of memory or time in the case that $T$ is beyond a certain threshold. Particularly, NetHash maintains the same level as \#GNN when $T=1$, but performs much more slowly than \#GNN as $T$ increases if it could give the results within the valid time. Furthermore, in the original DBLP network, \#GNN still scales linearly with $T$ increasing, while BANE and LQANR fail, and NetHash runs only two iterations. This illustrates that practically, \#GNN has the capability of exploring high-order neighboring nodes in the large-scale network. The main reason is that the property of the Markov chain enables \#GNN to generate node representation in the next iteration just based on the present node representation.

\subsection{Parameter Sensitivity}

The \#GNN algorithm has two parameters, the number of iterations $T$ and the embedding size $K$, and we study the impact of the two parameters on the link prediction performance in terms of accuracy and end-to-end runtime. We report the results on the five datasets in the case of $90\%$ training ratio in Figure \ref{fig:parameters}.

The accuracy performance is tightly related to the number of iterations $T$ and the embedding size $K$. In most cases, as $T$ and $K$ increases, the accuracy results get improved because the algorithm captures more information. However, the larger $T$ and $K$ might deteriorate the accuracy performance if noise dominates in the message passing process. The phenomenon illustrates that the quality of the node representation produced by the \#GNN algorithm is data-specific and deserves to be appropriately generated by tuning parameters. We observe from Tables \ref{tab:prediction} and \ref{tab:embedding} that the time for similarity computation could be ignored, 
and thus it is not surprising that the end-to-end runtime in the cases of $K$ generally linearly increases w.r.t. $T$.

\begin{figure*}[t]
\setlength{\abovecaptionskip}{5pt}%
\setlength{\belowcaptionskip}{0pt}%
\centering
\includegraphics[width=\linewidth]{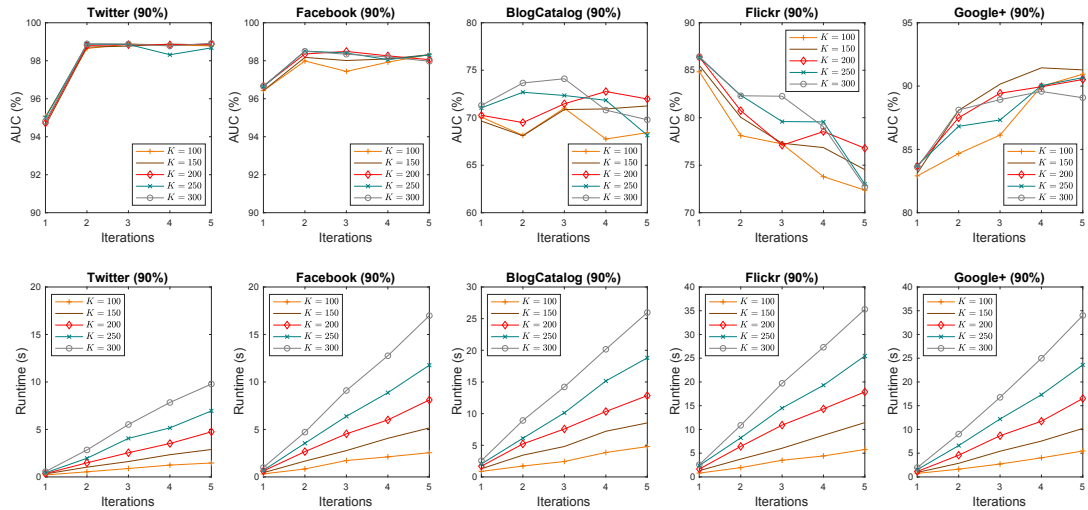}
  \caption{Impact of the number of iteraions $T$ and the embedding size $K$ on the link prediction performance in terms of accuracy (upper subplots) and runtime (lower subplots).}
\label{fig:parameters}
\end{figure*}

\section{Conclusion}
\label{sec:con}

In this paper, we propose an efficient and scalable link prediction algorithm dubbed \#GNN on the attributed network. The approach adopts the randomized hashing technique to avoid massive parameter learning and to accelerate the GNN-based link prediction model, while still preserving as much attribute and structure information as possible. In addition, we characterize the discriminative power of \#GNN in probability. 


We conduct extensive empirical tests of the proposed \#GNN algorithm and the state-of-the-art methods. We evaluate its effectiveness and efficiency on five network datasets and scalability on one large-scale network with millions of nodes. The experimental results show that \#GNN not only shows superiority in efficiency but also competes well with the compared algorithms. Additionally, \#GNN can scale linearly w.r.t. the size of the network and the number of iterations of the algorithm, which makes it more practical in the era of big data.

\section*{Acknowledgment}
This work was supported in part by the Federal Ministry of Education and Research (BMBF), Germany under the project LeibnizKILabor (Grant No.01DD20003), the Sub Project of Independent Scientific Research Project (No. ZZKY-ZX-03-02-04), STCSM Project (20511100400), Shanghai Municipal Science and Technology Major Projects (2018SHZDZX01), and the Program for Professor of Special Appointment (Eastern Scholar) at Shanghai Institutions of Higher Learning.

\bibliographystyle{ACM-Reference-Format}
\bibliography{main}


\begin{thebibliography}{57}


\ifx \showCODEN    \undefined \def \showCODEN     #1{\unskip}     \fi
\ifx \showDOI      \undefined \def \showDOI       #1{#1}\fi
\ifx \showISBNx    \undefined \def \showISBNx     #1{\unskip}     \fi
\ifx \showISBNxiii \undefined \def \showISBNxiii  #1{\unskip}     \fi
\ifx \showISSN     \undefined \def \showISSN      #1{\unskip}     \fi
\ifx \showLCCN     \undefined \def \showLCCN      #1{\unskip}     \fi
\ifx \shownote     \undefined \def \shownote      #1{#1}          \fi
\ifx \showarticletitle \undefined \def \showarticletitle #1{#1}   \fi
\ifx \showURL      \undefined \def \showURL       {\relax}        \fi
\providecommand\bibfield[2]{#2}
\providecommand\bibinfo[2]{#2}
\providecommand\natexlab[1]{#1}
\providecommand\showeprint[2][]{arXiv:#2}

\bibitem[\protect\citeauthoryear{Adamic and Adar}{Adamic and Adar}{2003}]%
        {adamic2003friends}
\bibfield{author}{\bibinfo{person}{Lada~A Adamic} {and} \bibinfo{person}{Eytan
  Adar}.} \bibinfo{year}{2003}\natexlab{}.
\newblock \showarticletitle{{F}riends and {N}eighbors on the {W}eb}.
\newblock \bibinfo{journal}{\emph{Social networks}} \bibinfo{volume}{25},
  \bibinfo{number}{3} (\bibinfo{year}{2003}), \bibinfo{pages}{211--230}.
\newblock


\bibitem[\protect\citeauthoryear{Aggarwal}{Aggarwal}{2011}]%
        {aggarwal2011classification}
\bibfield{author}{\bibinfo{person}{Charu~C Aggarwal}.}
  \bibinfo{year}{2011}\natexlab{}.
\newblock \showarticletitle{{O}n {C}lassification of {G}raph {S}treams}. In
  \bibinfo{booktitle}{\emph{SDM}}. \bibinfo{pages}{652--663}.
\newblock


\bibitem[\protect\citeauthoryear{Broder, Charikar, Frieze, and
  Mitzenmacher}{Broder et~al\mbox{.}}{1998}]%
        {broder1998min}
\bibfield{author}{\bibinfo{person}{Andrei~Z Broder}, \bibinfo{person}{Moses
  Charikar}, \bibinfo{person}{Alan~M Frieze}, {and} \bibinfo{person}{Michael
  Mitzenmacher}.} \bibinfo{year}{1998}\natexlab{}.
\newblock \showarticletitle{{M}in-wise {I}ndependent {P}ermutations}. In
  \bibinfo{booktitle}{\emph{STOC}}. \bibinfo{pages}{327--336}.
\newblock


\bibitem[\protect\citeauthoryear{Charikar}{Charikar}{2002}]%
        {charikar2002similarity}
\bibfield{author}{\bibinfo{person}{Moses~S Charikar}.}
  \bibinfo{year}{2002}\natexlab{}.
\newblock \showarticletitle{{S}imilarity {E}stimation {T}echniques from
  {R}ounding {A}lgorithms}. In \bibinfo{booktitle}{\emph{STOC}}.
  \bibinfo{pages}{380--388}.
\newblock


\bibitem[\protect\citeauthoryear{Duvenaud, Maclaurin, Iparraguirre, Bombarell,
  Hirzel, Aspuru-Guzik, and Adams}{Duvenaud et~al\mbox{.}}{2015}]%
        {duvenaud2015convolutional}
\bibfield{author}{\bibinfo{person}{David~K Duvenaud}, \bibinfo{person}{Dougal
  Maclaurin}, \bibinfo{person}{Jorge Iparraguirre}, \bibinfo{person}{Rafael
  Bombarell}, \bibinfo{person}{Timothy Hirzel}, \bibinfo{person}{Al{\'a}n
  Aspuru-Guzik}, {and} \bibinfo{person}{Ryan~P Adams}.}
  \bibinfo{year}{2015}\natexlab{}.
\newblock \showarticletitle{{C}onvolutional {N}etworks on {G}raphs for
  {L}earning {M}olecular {F}ingerprints}. In \bibinfo{booktitle}{\emph{NIPS}}.
  \bibinfo{pages}{2224--2232}.
\newblock


\bibitem[\protect\citeauthoryear{Gilmer, Schoenholz, Riley, Vinyals, and
  Dahl}{Gilmer et~al\mbox{.}}{2017}]%
        {gilmer2017neural}
\bibfield{author}{\bibinfo{person}{Justin Gilmer}, \bibinfo{person}{Samuel~S.
  Schoenholz}, \bibinfo{person}{Patrick~F. Riley}, \bibinfo{person}{Oriol
  Vinyals}, {and} \bibinfo{person}{George~E. Dahl}.}
  \bibinfo{year}{2017}\natexlab{}.
\newblock \showarticletitle{{N}eural {M}essage {P}assing for {Q}uantum
  {C}hemistry}. In \bibinfo{booktitle}{\emph{ICML}}.
  \bibinfo{pages}{1263--1272}.
\newblock


\bibitem[\protect\citeauthoryear{Grover and Leskovec}{Grover and
  Leskovec}{2016}]%
        {grover2016node2vec}
\bibfield{author}{\bibinfo{person}{Aditya Grover} {and} \bibinfo{person}{Jure
  Leskovec}.} \bibinfo{year}{2016}\natexlab{}.
\newblock \showarticletitle{node2vec: {S}calable {F}eature {L}earning for
  {N}etworks}. In \bibinfo{booktitle}{\emph{KDD}}. \bibinfo{pages}{855--864}.
\newblock


\bibitem[\protect\citeauthoryear{Hamilton, Ying, and Leskovec}{Hamilton
  et~al\mbox{.}}{2017}]%
        {hamilton2017inductive}
\bibfield{author}{\bibinfo{person}{Will Hamilton}, \bibinfo{person}{Zhitao
  Ying}, {and} \bibinfo{person}{Jure Leskovec}.}
  \bibinfo{year}{2017}\natexlab{}.
\newblock \showarticletitle{{I}nductive {R}epresentation {L}earning on {L}arge
  {G}raphs}. In \bibinfo{booktitle}{\emph{NIPS}}. \bibinfo{pages}{1024--1034}.
\newblock


\bibitem[\protect\citeauthoryear{Jiang and Li}{Jiang and Li}{2015}]%
        {jiang2015scalable}
\bibfield{author}{\bibinfo{person}{Qing-Yuan Jiang} {and}
  \bibinfo{person}{Wu-Jun Li}.} \bibinfo{year}{2015}\natexlab{}.
\newblock \showarticletitle{{S}calable {G}raph {H}ashing with {F}eature
  {T}ransformation}. In \bibinfo{booktitle}{\emph{AAAI}}.
  \bibinfo{pages}{2248--2254}.
\newblock


\bibitem[\protect\citeauthoryear{Kearnes, McCloskey, Berndl, Pande, and
  Riley}{Kearnes et~al\mbox{.}}{2016}]%
        {kearnes2016molecular}
\bibfield{author}{\bibinfo{person}{Steven Kearnes}, \bibinfo{person}{Kevin
  McCloskey}, \bibinfo{person}{Marc Berndl}, \bibinfo{person}{Vijay Pande},
  {and} \bibinfo{person}{Patrick Riley}.} \bibinfo{year}{2016}\natexlab{}.
\newblock \showarticletitle{{M}olecular {G}raph {C}onvolutions: {M}oving beyond
  {F}ingerprints}.
\newblock \bibinfo{journal}{\emph{Journal of computer-aided molecular design}}
  \bibinfo{volume}{30}, \bibinfo{number}{8} (\bibinfo{year}{2016}),
  \bibinfo{pages}{595--608}.
\newblock


\bibitem[\protect\citeauthoryear{Kipf and Welling}{Kipf and Welling}{2017}]%
        {kipf2016semi}
\bibfield{author}{\bibinfo{person}{Thomas~N Kipf} {and} \bibinfo{person}{Max
  Welling}.} \bibinfo{year}{2017}\natexlab{}.
\newblock \showarticletitle{{S}emi-{S}upervised {C}lassification with {G}raph
  {C}onvolutional {N}etworks}. In \bibinfo{booktitle}{\emph{ICLR}}.
\newblock


\bibitem[\protect\citeauthoryear{Kontorovich, Ramanan,
  et~al\mbox{.}}{Kontorovich et~al\mbox{.}}{2008}]%
        {kontorovich2008concentration}
\bibfield{author}{\bibinfo{person}{Leonid~Aryeh Kontorovich},
  \bibinfo{person}{Kavita Ramanan}, {et~al\mbox{.}}}
  \bibinfo{year}{2008}\natexlab{}.
\newblock \showarticletitle{{C}oncentration {I}nequalities for {D}ependent
  {R}andom {V}ariables via the {M}artingale {M}ethod}.
\newblock \bibinfo{journal}{\emph{The Annals of Probability}}
  \bibinfo{volume}{36}, \bibinfo{number}{6} (\bibinfo{year}{2008}),
  \bibinfo{pages}{2126--2158}.
\newblock


\bibitem[\protect\citeauthoryear{Koren, Bell, and Volinsky}{Koren
  et~al\mbox{.}}{2009}]%
        {koren2009matrix}
\bibfield{author}{\bibinfo{person}{Yehuda Koren}, \bibinfo{person}{Robert
  Bell}, {and} \bibinfo{person}{Chris Volinsky}.}
  \bibinfo{year}{2009}\natexlab{}.
\newblock \showarticletitle{{M}atrix {F}actorization {T}echniques for
  {R}ecommender {S}ystems}.
\newblock \bibinfo{journal}{\emph{Computer}} \bibinfo{volume}{42},
  \bibinfo{number}{8} (\bibinfo{year}{2009}), \bibinfo{pages}{30--37}.
\newblock


\bibitem[\protect\citeauthoryear{Lei, Jin, Barzilay, and Jaakkola}{Lei
  et~al\mbox{.}}{2017}]%
        {lei2017deriving}
\bibfield{author}{\bibinfo{person}{Tao Lei}, \bibinfo{person}{Wengong Jin},
  \bibinfo{person}{Regina Barzilay}, {and} \bibinfo{person}{Tommi Jaakkola}.}
  \bibinfo{year}{2017}\natexlab{}.
\newblock \showarticletitle{{D}eriving {N}eural {A}rchitectures from {S}equence
  and {G}raph {K}ernels}. In \bibinfo{booktitle}{\emph{ICML}}.
  \bibinfo{pages}{2024--2033}.
\newblock


\bibitem[\protect\citeauthoryear{Leskovec and Mcauley}{Leskovec and
  Mcauley}{2012}]%
        {leskovec2012learning}
\bibfield{author}{\bibinfo{person}{Jure Leskovec} {and}
  \bibinfo{person}{Julian~J Mcauley}.} \bibinfo{year}{2012}\natexlab{}.
\newblock \showarticletitle{{L}earning to {D}iscover {S}ocial {C}ircles in
  {E}go {N}etworks}. In \bibinfo{booktitle}{\emph{NIPS}}.
  \bibinfo{pages}{539--547}.
\newblock


\bibitem[\protect\citeauthoryear{Li, Zhu, Chi, and Zhang}{Li
  et~al\mbox{.}}{2012}]%
        {li2012nest}
\bibfield{author}{\bibinfo{person}{Bin Li}, \bibinfo{person}{Xingquan Zhu},
  \bibinfo{person}{Lianhua Chi}, {and} \bibinfo{person}{Chengqi Zhang}.}
  \bibinfo{year}{2012}\natexlab{}.
\newblock \showarticletitle{{N}ested {S}ubtree {H}ash {K}ernels for
  {L}arge-scale {G}raph {C}lassification over {S}treams}. In
  \bibinfo{booktitle}{\emph{ICDM}}. \bibinfo{pages}{399--408}.
\newblock


\bibitem[\protect\citeauthoryear{Li, Hu, Tang, and Liu}{Li
  et~al\mbox{.}}{2015}]%
        {li2015unsupervised}
\bibfield{author}{\bibinfo{person}{Jundong Li}, \bibinfo{person}{Xia Hu},
  \bibinfo{person}{Jiliang Tang}, {and} \bibinfo{person}{Huan Liu}.}
  \bibinfo{year}{2015}\natexlab{}.
\newblock \showarticletitle{{U}nsupervised {S}treaming {F}eature {S}election in
  {S}ocial {M}edia}. In \bibinfo{booktitle}{\emph{CIKM}}.
  \bibinfo{pages}{1041--1050}.
\newblock


\bibitem[\protect\citeauthoryear{Li, Tarlow, Brockschmidt, and Zemel}{Li
  et~al\mbox{.}}{2016}]%
        {li2016gated}
\bibfield{author}{\bibinfo{person}{Yujia Li}, \bibinfo{person}{Daniel Tarlow},
  \bibinfo{person}{Marc Brockschmidt}, {and} \bibinfo{person}{Richard~S.
  Zemel}.} \bibinfo{year}{2016}\natexlab{}.
\newblock \showarticletitle{{G}ated {G}raph {S}equence {N}eural {N}etworks}. In
  \bibinfo{booktitle}{\emph{ICLR}}.
\newblock


\bibitem[\protect\citeauthoryear{Lian, Zheng, Zheng, Ge, Cao, Tsang, and
  Xie}{Lian et~al\mbox{.}}{2018}]%
        {lian2018high}
\bibfield{author}{\bibinfo{person}{Defu Lian}, \bibinfo{person}{Kai Zheng},
  \bibinfo{person}{Vincent~W Zheng}, \bibinfo{person}{Yong Ge},
  \bibinfo{person}{Longbing Cao}, \bibinfo{person}{Ivor~W Tsang}, {and}
  \bibinfo{person}{Xing Xie}.} \bibinfo{year}{2018}\natexlab{}.
\newblock \showarticletitle{High-order {P}roximity {P}reserving {I}nformation
  {N}etwork {H}ashing}. In \bibinfo{booktitle}{\emph{SIGKDD}}.
  \bibinfo{pages}{1744--1753}.
\newblock


\bibitem[\protect\citeauthoryear{Liben-Nowell and Kleinberg}{Liben-Nowell and
  Kleinberg}{2007}]%
        {liben2007link}
\bibfield{author}{\bibinfo{person}{David Liben-Nowell} {and}
  \bibinfo{person}{Jon Kleinberg}.} \bibinfo{year}{2007}\natexlab{}.
\newblock \showarticletitle{{T}he {L}ink-prediction {P}roblem for {S}ocial
  {N}etworks}.
\newblock \bibinfo{journal}{\emph{Journal of the Association for Information
  Science and Technology}} \bibinfo{volume}{58}, \bibinfo{number}{7}
  (\bibinfo{year}{2007}), \bibinfo{pages}{1019--1031}.
\newblock


\bibitem[\protect\citeauthoryear{Liu, Mu, Kumar, and Chang}{Liu
  et~al\mbox{.}}{2014}]%
        {liu2014discrete}
\bibfield{author}{\bibinfo{person}{Wei Liu}, \bibinfo{person}{Cun Mu},
  \bibinfo{person}{Sanjiv Kumar}, {and} \bibinfo{person}{Shih-Fu Chang}.}
  \bibinfo{year}{2014}\natexlab{}.
\newblock \showarticletitle{{D}iscrete {G}raph {H}ashing}. In
  \bibinfo{booktitle}{\emph{NIPS}}. \bibinfo{pages}{3419--3427}.
\newblock


\bibitem[\protect\citeauthoryear{Manku, Jain, and Das~Sarma}{Manku
  et~al\mbox{.}}{2007}]%
        {manku2007detecting}
\bibfield{author}{\bibinfo{person}{Gurmeet~Singh Manku},
  \bibinfo{person}{Arvind Jain}, {and} \bibinfo{person}{Anish Das~Sarma}.}
  \bibinfo{year}{2007}\natexlab{}.
\newblock \showarticletitle{{D}etecting {N}ear-duplicates for {W}eb
  {C}rawling}. In \bibinfo{booktitle}{\emph{WWW}}. \bibinfo{pages}{141--150}.
\newblock


\bibitem[\protect\citeauthoryear{Mikolov, Sutskever, Chen, Corrado, and
  Dean}{Mikolov et~al\mbox{.}}{2013}]%
        {mikolov2013distributed}
\bibfield{author}{\bibinfo{person}{Tomas Mikolov}, \bibinfo{person}{Ilya
  Sutskever}, \bibinfo{person}{Kai Chen}, \bibinfo{person}{Greg~S Corrado},
  {and} \bibinfo{person}{Jeff Dean}.} \bibinfo{year}{2013}\natexlab{}.
\newblock \showarticletitle{{D}istributed {R}epresentations of {W}ords and
  {P}hrases and {T}heir {C}ompositionality}. In
  \bibinfo{booktitle}{\emph{NIPS}}. \bibinfo{pages}{3111--3119}.
\newblock


\bibitem[\protect\citeauthoryear{Morris, Ritzert, Fey, Hamilton, Lenssen,
  Rattan, and Grohe}{Morris et~al\mbox{.}}{2019}]%
        {morris2019weisfeiler}
\bibfield{author}{\bibinfo{person}{Christopher Morris}, \bibinfo{person}{Martin
  Ritzert}, \bibinfo{person}{Matthias Fey}, \bibinfo{person}{William~L
  Hamilton}, \bibinfo{person}{Jan~Eric Lenssen}, \bibinfo{person}{Gaurav
  Rattan}, {and} \bibinfo{person}{Martin Grohe}.}
  \bibinfo{year}{2019}\natexlab{}.
\newblock \showarticletitle{{W}eisfeiler and {L}eman {G}o {N}eural:
  {H}igher-order {G}raph {N}eural {N}etworks}. In
  \bibinfo{booktitle}{\emph{AAAI}}. \bibinfo{pages}{4602--4609}.
\newblock


\bibitem[\protect\citeauthoryear{Nickel, Murphy, Tresp, and Gabrilovich}{Nickel
  et~al\mbox{.}}{2015}]%
        {nickel2015review}
\bibfield{author}{\bibinfo{person}{Maximilian Nickel}, \bibinfo{person}{Kevin
  Murphy}, \bibinfo{person}{Volker Tresp}, {and} \bibinfo{person}{Evgeniy
  Gabrilovich}.} \bibinfo{year}{2015}\natexlab{}.
\newblock \showarticletitle{{A} {R}eview of {R}elational {M}achine {L}earning
  for {K}nowledge {G}raphs}.
\newblock \bibinfo{journal}{\emph{Proc. IEEE}} \bibinfo{volume}{104},
  \bibinfo{number}{1} (\bibinfo{year}{2015}), \bibinfo{pages}{11--33}.
\newblock


\bibitem[\protect\citeauthoryear{Oyetunde, Zhang, Chen, Tang, and Lo}{Oyetunde
  et~al\mbox{.}}{2017}]%
        {oyetunde2017boostgapfill}
\bibfield{author}{\bibinfo{person}{Tolutola Oyetunde}, \bibinfo{person}{Muhan
  Zhang}, \bibinfo{person}{Yixin Chen}, \bibinfo{person}{Yinjie Tang}, {and}
  \bibinfo{person}{Cynthia Lo}.} \bibinfo{year}{2017}\natexlab{}.
\newblock \showarticletitle{{B}oostgapfill: {I}mproving the {F}idelity of
  {M}etabolic {N}etwork {R}econstructions through {I}ntegrated {C}onstraint and
  {P}attern-based {M}ethods}.
\newblock \bibinfo{journal}{\emph{Bioinformatics}} \bibinfo{volume}{33},
  \bibinfo{number}{4} (\bibinfo{year}{2017}), \bibinfo{pages}{608--611}.
\newblock


\bibitem[\protect\citeauthoryear{Perozzi, Al-Rfou, and Skiena}{Perozzi
  et~al\mbox{.}}{2014}]%
        {perozzi2014deepwalk}
\bibfield{author}{\bibinfo{person}{Bryan Perozzi}, \bibinfo{person}{Rami
  Al-Rfou}, {and} \bibinfo{person}{Steven Skiena}.}
  \bibinfo{year}{2014}\natexlab{}.
\newblock \showarticletitle{{D}eep{W}alk: {O}nline {L}earning of {S}ocial
  {R}epresentations}. In \bibinfo{booktitle}{\emph{KDD}}.
  \bibinfo{pages}{701--710}.
\newblock


\bibitem[\protect\citeauthoryear{Qiu, Dong, Ma, Li, Wang, and Tang}{Qiu
  et~al\mbox{.}}{2018}]%
        {qiu2018network}
\bibfield{author}{\bibinfo{person}{Jiezhong Qiu}, \bibinfo{person}{Yuxiao
  Dong}, \bibinfo{person}{Hao Ma}, \bibinfo{person}{Jian Li},
  \bibinfo{person}{Kuansan Wang}, {and} \bibinfo{person}{Jie Tang}.}
  \bibinfo{year}{2018}\natexlab{}.
\newblock \showarticletitle{{N}etwork {E}mbedding as {M}atrix {F}actorization:
  {U}nifying {D}eep{W}alk, {LINE}, {PTE}, and node2vec}. In
  \bibinfo{booktitle}{\emph{WSDM}}. \bibinfo{pages}{459--467}.
\newblock


\bibitem[\protect\citeauthoryear{Salakhutdinov and Hinton}{Salakhutdinov and
  Hinton}{2009}]%
        {salakhutdinov2009semantic}
\bibfield{author}{\bibinfo{person}{Ruslan Salakhutdinov} {and}
  \bibinfo{person}{Geoffrey Hinton}.} \bibinfo{year}{2009}\natexlab{}.
\newblock \showarticletitle{{S}emantic {H}ashing}.
\newblock \bibinfo{journal}{\emph{International Journal of Approximate
  Reasoning}} \bibinfo{volume}{50}, \bibinfo{number}{7} (\bibinfo{year}{2009}),
  \bibinfo{pages}{969--978}.
\newblock


\bibitem[\protect\citeauthoryear{Sch{\"u}tt, Arbabzadah, Chmiela, M{\"u}ller,
  and Tkatchenko}{Sch{\"u}tt et~al\mbox{.}}{2017}]%
        {schutt2017quantum}
\bibfield{author}{\bibinfo{person}{Kristof~T Sch{\"u}tt},
  \bibinfo{person}{Farhad Arbabzadah}, \bibinfo{person}{Stefan Chmiela},
  \bibinfo{person}{Klaus~R M{\"u}ller}, {and} \bibinfo{person}{Alexandre
  Tkatchenko}.} \bibinfo{year}{2017}\natexlab{}.
\newblock \showarticletitle{{Q}uantum-chemical {I}nsights from {D}eep {T}ensor
  {N}eural {N}etworks}.
\newblock \bibinfo{journal}{\emph{Nature communications}} \bibinfo{volume}{8},
  \bibinfo{number}{1} (\bibinfo{year}{2017}), \bibinfo{pages}{1--8}.
\newblock


\bibitem[\protect\citeauthoryear{Shen, Cao, Zou, and Wang}{Shen
  et~al\mbox{.}}{2018}]%
        {shen2018flexible}
\bibfield{author}{\bibinfo{person}{Enya Shen}, \bibinfo{person}{Zhidong Cao},
  \bibinfo{person}{Changqing Zou}, {and} \bibinfo{person}{Jianmin Wang}.}
  \bibinfo{year}{2018}\natexlab{}.
\newblock \bibinfo{title}{{F}lexible {A}ttributed {N}etwork {E}mbedding}.
\newblock
\newblock
\showeprint[arxiv]{1811.10789}~[cs.SI]


\bibitem[\protect\citeauthoryear{Shervashidze, Schweitzer, Van~Leeuwen,
  Mehlhorn, and Borgwardt}{Shervashidze et~al\mbox{.}}{2011}]%
        {shervashidze2011weisfeiler}
\bibfield{author}{\bibinfo{person}{Nino Shervashidze}, \bibinfo{person}{Pascal
  Schweitzer}, \bibinfo{person}{Erik~Jan Van~Leeuwen}, \bibinfo{person}{Kurt
  Mehlhorn}, {and} \bibinfo{person}{Karsten~M Borgwardt}.}
  \bibinfo{year}{2011}\natexlab{}.
\newblock \showarticletitle{{W}eisfeiler-{L}ehman {G}raph {K}ernels.}
\newblock \bibinfo{journal}{\emph{Journal of Machine Learning Research}}
  \bibinfo{volume}{12}, \bibinfo{number}{9} (\bibinfo{year}{2011}).
\newblock


\bibitem[\protect\citeauthoryear{Shi, Xing, Xu, Sapkota, and Yang}{Shi
  et~al\mbox{.}}{2017}]%
        {shi2017asymmetric}
\bibfield{author}{\bibinfo{person}{Xiaoshuang Shi}, \bibinfo{person}{Fuyong
  Xing}, \bibinfo{person}{Kaidi Xu}, \bibinfo{person}{Manish Sapkota}, {and}
  \bibinfo{person}{Lin Yang}.} \bibinfo{year}{2017}\natexlab{}.
\newblock \showarticletitle{{A}symmetric {D}iscrete {G}raph {H}ashing}. In
  \bibinfo{booktitle}{\emph{AAAI}}. \bibinfo{pages}{2541--2547}.
\newblock


\bibitem[\protect\citeauthoryear{Tan, Liu, Zhao, Yang, Zhou, and Hu}{Tan
  et~al\mbox{.}}{2020}]%
        {tan2020learning}
\bibfield{author}{\bibinfo{person}{Qiaoyu Tan}, \bibinfo{person}{Ninghao Liu},
  \bibinfo{person}{Xing Zhao}, \bibinfo{person}{Hongxia Yang},
  \bibinfo{person}{Jingren Zhou}, {and} \bibinfo{person}{Xia Hu}.}
  \bibinfo{year}{2020}\natexlab{}.
\newblock \showarticletitle{{L}earning to {H}ash with {G}raph {N}eural
  {N}etworks for {R}ecommender {S}ystems}. In \bibinfo{booktitle}{\emph{WWW}}.
  \bibinfo{pages}{1988--1998}.
\newblock


\bibitem[\protect\citeauthoryear{Tang, Qu, Wang, Zhang, Yan, and Mei}{Tang
  et~al\mbox{.}}{2015}]%
        {tang2015line}
\bibfield{author}{\bibinfo{person}{Jian Tang}, \bibinfo{person}{Meng Qu},
  \bibinfo{person}{Mingzhe Wang}, \bibinfo{person}{Ming Zhang},
  \bibinfo{person}{Jun Yan}, {and} \bibinfo{person}{Qiaozhu Mei}.}
  \bibinfo{year}{2015}\natexlab{}.
\newblock \showarticletitle{{LINE}: {L}arge-scale {I}nformation {N}etwork
  {E}mbedding}. In \bibinfo{booktitle}{\emph{WWW}}.
  \bibinfo{pages}{1067--1077}.
\newblock


\bibitem[\protect\citeauthoryear{Tang, Zhang, Yao, Li, Zhang, and Su}{Tang
  et~al\mbox{.}}{2008}]%
        {tang2008arnetminer}
\bibfield{author}{\bibinfo{person}{Jie Tang}, \bibinfo{person}{Jing Zhang},
  \bibinfo{person}{Limin Yao}, \bibinfo{person}{Juanzi Li}, \bibinfo{person}{Li
  Zhang}, {and} \bibinfo{person}{Zhong Su}.} \bibinfo{year}{2008}\natexlab{}.
\newblock \showarticletitle{{A}rnet{M}iner: {E}xtraction and {M}ining of
  {A}cademic {S}ocial {N}etworks}. In \bibinfo{booktitle}{\emph{KDD}}.
  \bibinfo{pages}{990--998}.
\newblock


\bibitem[\protect\citeauthoryear{Tu, Liu, Liu, and Sun}{Tu
  et~al\mbox{.}}{2017}]%
        {tu2017cane}
\bibfield{author}{\bibinfo{person}{Cunchao Tu}, \bibinfo{person}{Han Liu},
  \bibinfo{person}{Zhiyuan Liu}, {and} \bibinfo{person}{Maosong Sun}.}
  \bibinfo{year}{2017}\natexlab{}.
\newblock \showarticletitle{{C}ane: {C}ontext-{A}ware {N}etwork {E}mbedding for
  {R}elation {M}odeling}. In \bibinfo{booktitle}{\emph{ACL}},
  Vol.~\bibinfo{volume}{1}. \bibinfo{pages}{1722--1731}.
\newblock


\bibitem[\protect\citeauthoryear{Weinberger, Dasgupta, Langford, Smola, and
  Attenberg}{Weinberger et~al\mbox{.}}{2009}]%
        {weinberger2009feature}
\bibfield{author}{\bibinfo{person}{Kilian Weinberger}, \bibinfo{person}{Anirban
  Dasgupta}, \bibinfo{person}{John Langford}, \bibinfo{person}{Alex Smola},
  {and} \bibinfo{person}{Josh Attenberg}.} \bibinfo{year}{2009}\natexlab{}.
\newblock \showarticletitle{{F}eature {H}ashing for {L}arge {S}cale {M}ultitask
  {L}earning}. In \bibinfo{booktitle}{\emph{ICML}}.
  \bibinfo{pages}{1113--1120}.
\newblock


\bibitem[\protect\citeauthoryear{Weisfeiler and Lehman}{Weisfeiler and
  Lehman}{1968}]%
        {weisfeiler1968reduction}
\bibfield{author}{\bibinfo{person}{Boris Weisfeiler} {and}
  \bibinfo{person}{Andrei~A Lehman}.} \bibinfo{year}{1968}\natexlab{}.
\newblock \showarticletitle{{A} {R}eduction of a {G}raph to a {C}anonical
  {F}orm and an {A}lgebra {A}rising {D}uring this {R}eduction}.
\newblock \bibinfo{journal}{\emph{Nauchno-Technicheskaya Informatsia}}
  \bibinfo{volume}{2}, \bibinfo{number}{9} (\bibinfo{year}{1968}),
  \bibinfo{pages}{12--16}.
\newblock


\bibitem[\protect\citeauthoryear{Weiss, Torralba, and Fergus}{Weiss
  et~al\mbox{.}}{2009}]%
        {weiss2009spectral}
\bibfield{author}{\bibinfo{person}{Yair Weiss}, \bibinfo{person}{Antonio
  Torralba}, {and} \bibinfo{person}{Rob Fergus}.}
  \bibinfo{year}{2009}\natexlab{}.
\newblock \showarticletitle{{S}pectral {H}ashing}. In
  \bibinfo{booktitle}{\emph{NIPS}}. \bibinfo{pages}{1753--1760}.
\newblock


\bibitem[\protect\citeauthoryear{Wu, Li, Chen, Gao, and Zhang}{Wu
  et~al\mbox{.}}{2020}]%
        {wu2018review}
\bibfield{author}{\bibinfo{person}{Wei Wu}, \bibinfo{person}{Bin Li},
  \bibinfo{person}{Ling Chen}, \bibinfo{person}{Junbin Gao}, {and}
  \bibinfo{person}{Chengqi Zhang}.} \bibinfo{year}{2020}\natexlab{}.
\newblock \showarticletitle{{A} {R}eview for {W}eighted {M}inHash
  {A}lgorithms}.
\newblock \bibinfo{journal}{\emph{IEEE Transactions on Knowledge and Data
  Engineering}} (\bibinfo{year}{2020}), \bibinfo{pages}{1--1}.
\newblock


\bibitem[\protect\citeauthoryear{Wu, Li, Chen, and Zhang}{Wu
  et~al\mbox{.}}{2016}]%
        {wu2016canonical}
\bibfield{author}{\bibinfo{person}{Wei Wu}, \bibinfo{person}{Bin Li},
  \bibinfo{person}{Ling Chen}, {and} \bibinfo{person}{Chengqi Zhang}.}
  \bibinfo{year}{2016}\natexlab{}.
\newblock \showarticletitle{{C}anonical {C}onsistent {W}eighted {S}ampling for
  {R}eal-{V}alue {W}eighted {M}in-{H}ash}. In \bibinfo{booktitle}{\emph{ICDM}}.
  \bibinfo{pages}{1287--1292}.
\newblock


\bibitem[\protect\citeauthoryear{Wu, Li, Chen, and Zhang}{Wu
  et~al\mbox{.}}{2017}]%
        {wu2017consistent}
\bibfield{author}{\bibinfo{person}{Wei Wu}, \bibinfo{person}{Bin Li},
  \bibinfo{person}{Ling Chen}, {and} \bibinfo{person}{Chengqi Zhang}.}
  \bibinfo{year}{2017}\natexlab{}.
\newblock \showarticletitle{{C}onsistent {W}eighted {S}ampling {Ma}de {M}ore
  {P}ractical}. In \bibinfo{booktitle}{\emph{WWW}}.
  \bibinfo{pages}{1035--1043}.
\newblock


\bibitem[\protect\citeauthoryear{Wu, Li, Chen, and Zhang}{Wu
  et~al\mbox{.}}{2018a}]%
        {wu2018efficient}
\bibfield{author}{\bibinfo{person}{Wei Wu}, \bibinfo{person}{Bin Li},
  \bibinfo{person}{Ling Chen}, {and} \bibinfo{person}{Chengqi Zhang}.}
  \bibinfo{year}{2018}\natexlab{a}.
\newblock \showarticletitle{{E}fficient {A}ttributed {N}etwork {E}mbedding via
  {R}ecursive {R}andomized {H}ashing}. In \bibinfo{booktitle}{\emph{IJCAI-18}}.
  \bibinfo{pages}{2861--2867}.
\newblock


\bibitem[\protect\citeauthoryear{Wu, Li, Chen, Zhang, and Philip}{Wu
  et~al\mbox{.}}{2018b}]%
        {wu2018improved}
\bibfield{author}{\bibinfo{person}{Wei Wu}, \bibinfo{person}{Bin Li},
  \bibinfo{person}{Ling Chen}, \bibinfo{person}{Chengqi Zhang}, {and}
  \bibinfo{person}{S~Yu Philip}.} \bibinfo{year}{2018}\natexlab{b}.
\newblock \showarticletitle{{I}mproved {C}onsistent {W}eighted {S}ampling
  {R}evisited}.
\newblock \bibinfo{journal}{\emph{IEEE Transactions on Knowledge and Data
  Engineering}} \bibinfo{volume}{31}, \bibinfo{number}{12}
  (\bibinfo{year}{2018}), \bibinfo{pages}{2332--2345}.
\newblock


\bibitem[\protect\citeauthoryear{Wu, Li, Chen, Zhu, and Zhang}{Wu
  et~al\mbox{.}}{2018c}]%
        {wu2017k}
\bibfield{author}{\bibinfo{person}{Wei Wu}, \bibinfo{person}{Bin Li},
  \bibinfo{person}{Ling Chen}, \bibinfo{person}{Xingquan Zhu}, {and}
  \bibinfo{person}{Chengqi Zhang}.} \bibinfo{year}{2018}\natexlab{c}.
\newblock \showarticletitle{{$K$}-{A}ry {T}ree {H}ashing for {F}ast {G}raph
  {C}lassification}.
\newblock \bibinfo{journal}{\emph{IEEE Transactions on Knowledge and Data
  Engineering}} \bibinfo{volume}{30}, \bibinfo{number}{5}
  (\bibinfo{year}{2018}), \bibinfo{pages}{936--949}.
\newblock


\bibitem[\protect\citeauthoryear{Xu, Hu, Leskovec, and Jegelka}{Xu
  et~al\mbox{.}}{2019}]%
        {xu2018how}
\bibfield{author}{\bibinfo{person}{Keyulu Xu}, \bibinfo{person}{Weihua Hu},
  \bibinfo{person}{Jure Leskovec}, {and} \bibinfo{person}{Stefanie Jegelka}.}
  \bibinfo{year}{2019}\natexlab{}.
\newblock \showarticletitle{{H}ow {P}owerful are {G}raph {N}eural {N}etworks?}.
  In \bibinfo{booktitle}{\emph{ICLR}}.
\newblock


\bibitem[\protect\citeauthoryear{Yang, Liu, Zhao, Sun, and Chang}{Yang
  et~al\mbox{.}}{2015}]%
        {yang2015network}
\bibfield{author}{\bibinfo{person}{Cheng Yang}, \bibinfo{person}{Zhiyuan Liu},
  \bibinfo{person}{Deli Zhao}, \bibinfo{person}{Maosong Sun}, {and}
  \bibinfo{person}{Edward~Y Chang}.} \bibinfo{year}{2015}\natexlab{}.
\newblock \showarticletitle{{N}etwork {R}epresentation {L}earning with {R}ich
  {T}ext {I}nformation.}. In \bibinfo{booktitle}{\emph{IJCAI}}.
  \bibinfo{pages}{2111--2117}.
\newblock


\bibitem[\protect\citeauthoryear{Yang, Rosso, Li, and Cudre-Mauroux}{Yang
  et~al\mbox{.}}{2019b}]%
        {yang2019nodesketch}
\bibfield{author}{\bibinfo{person}{Dingqi Yang}, \bibinfo{person}{Paolo Rosso},
  \bibinfo{person}{Bin Li}, {and} \bibinfo{person}{Philippe Cudre-Mauroux}.}
  \bibinfo{year}{2019}\natexlab{b}.
\newblock \showarticletitle{{N}ode{S}ketch: {H}ighly-{E}fficient {G}raph
  {E}mbeddings via {R}ecursive {S}ketching}. In
  \bibinfo{booktitle}{\emph{KDD}}. \bibinfo{pages}{1162--1172}.
\newblock


\bibitem[\protect\citeauthoryear{Yang, Pan, Chen, Zhou, and Zhang}{Yang
  et~al\mbox{.}}{2019a}]%
        {yang2019low}
\bibfield{author}{\bibinfo{person}{Hong Yang}, \bibinfo{person}{Shirui Pan},
  \bibinfo{person}{Ling Chen}, \bibinfo{person}{Chuan Zhou}, {and}
  \bibinfo{person}{Peng Zhang}.} \bibinfo{year}{2019}\natexlab{a}.
\newblock \showarticletitle{{L}ow-{B}it {Q}uantization for {A}ttributed
  {N}etwork {R}epresentation {L}earning}. In \bibinfo{booktitle}{\emph{IJCAI}}.
  \bibinfo{pages}{4047--4053}.
\newblock


\bibitem[\protect\citeauthoryear{Yang, Pan, Zhang, Chen, Lian, and Zhang}{Yang
  et~al\mbox{.}}{2018}]%
        {yang2018binarized}
\bibfield{author}{\bibinfo{person}{Hong Yang}, \bibinfo{person}{Shirui Pan},
  \bibinfo{person}{Peng Zhang}, \bibinfo{person}{Ling Chen},
  \bibinfo{person}{Defu Lian}, {and} \bibinfo{person}{Chengqi Zhang}.}
  \bibinfo{year}{2018}\natexlab{}.
\newblock \showarticletitle{{B}inarized {A}ttributed {N}etwork {E}mbedding}. In
  \bibinfo{booktitle}{\emph{ICDM}}. \bibinfo{pages}{1476--1481}.
\newblock


\bibitem[\protect\citeauthoryear{You, Ying, and Leskovec}{You
  et~al\mbox{.}}{2019}]%
        {you2019position}
\bibfield{author}{\bibinfo{person}{Jiaxuan You}, \bibinfo{person}{Rex Ying},
  {and} \bibinfo{person}{Jure Leskovec}.} \bibinfo{year}{2019}\natexlab{}.
\newblock \showarticletitle{{P}osition-aware {G}raph {N}eural {N}etworks}. In
  \bibinfo{booktitle}{\emph{ICML}}. \bibinfo{pages}{7134--7143}.
\newblock


\bibitem[\protect\citeauthoryear{Zhang, Yin, Zhu, and Zhang}{Zhang
  et~al\mbox{.}}{2016}]%
        {zhang2016homophily}
\bibfield{author}{\bibinfo{person}{Daokun Zhang}, \bibinfo{person}{Jie Yin},
  \bibinfo{person}{Xingquan Zhu}, {and} \bibinfo{person}{Chengqi Zhang}.}
  \bibinfo{year}{2016}\natexlab{}.
\newblock \showarticletitle{{H}omophily, {S}tructure, and {C}ontent {A}ugmented
  {N}etwork {R}epresentation {L}earning}. In \bibinfo{booktitle}{\emph{ICDM}}.
  \bibinfo{pages}{609--618}.
\newblock


\bibitem[\protect\citeauthoryear{Zhang, Xu, Arnab, and Torr}{Zhang
  et~al\mbox{.}}{2020}]%
        {zhang2020dynamic}
\bibfield{author}{\bibinfo{person}{Li Zhang}, \bibinfo{person}{Dan Xu},
  \bibinfo{person}{Anurag Arnab}, {and} \bibinfo{person}{Philip~HS Torr}.}
  \bibinfo{year}{2020}\natexlab{}.
\newblock \showarticletitle{{D}ynamic {G}raph {M}essage {P}assing {N}etworks}.
  In \bibinfo{booktitle}{\emph{CVPR}}. \bibinfo{pages}{3726--3735}.
\newblock


\bibitem[\protect\citeauthoryear{Zhang and Chen}{Zhang and Chen}{2017}]%
        {zhang2017weisfeiler}
\bibfield{author}{\bibinfo{person}{Muhan Zhang} {and} \bibinfo{person}{Yixin
  Chen}.} \bibinfo{year}{2017}\natexlab{}.
\newblock \showarticletitle{{W}eisfeiler-{L}ehman {N}eural {M}achine for {L}ink
  {P}rediction}. In \bibinfo{booktitle}{\emph{KDD}}. \bibinfo{pages}{575--583}.
\newblock


\bibitem[\protect\citeauthoryear{Zhang and Chen}{Zhang and Chen}{2018}]%
        {zhang2018link}
\bibfield{author}{\bibinfo{person}{Muhan Zhang} {and} \bibinfo{person}{Yixin
  Chen}.} \bibinfo{year}{2018}\natexlab{}.
\newblock \showarticletitle{{L}ink {P}rediction based on {G}raph {N}eural
  {N}etworks}. In \bibinfo{booktitle}{\emph{NeurIPS}}.
  \bibinfo{pages}{5165--5175}.
\newblock


\bibitem[\protect\citeauthoryear{Zhou, While, and Kalogerakis}{Zhou
  et~al\mbox{.}}{2019}]%
        {zhou2019scenegraphnet}
\bibfield{author}{\bibinfo{person}{Yang Zhou}, \bibinfo{person}{Zachary While},
  {and} \bibinfo{person}{Evangelos Kalogerakis}.}
  \bibinfo{year}{2019}\natexlab{}.
\newblock \showarticletitle{{S}cene{G}raph{N}et: {N}eural {M}essage {P}assing
  for {3D} {I}ndoor {S}cene {A}ugmentation}. In
  \bibinfo{booktitle}{\emph{ICCV}}. \bibinfo{pages}{7384--7392}.
\newblock


\end{thebibliography}


\end{document}